\documentclass[11pt]{article}

\usepackage[top=1in, bottom=1in, left=1in, right=1in]{geometry}

\usepackage{latexsym}
\usepackage{epsfig}
\usepackage{bm}
\usepackage{tikz}
\usepackage{mathrsfs}
\usepackage[toc,page]{appendix}
\usepackage{bbm}
\usepackage{ytableau}
\usepackage{tkz-berge}
\usepackage[linesnumbered, ruled, vlined]{algorithm2e}
\usepackage{array}
\usepackage{multirow}
\renewcommand{\arraystretch}{2} 

\usepackage{hyperref}
\hypersetup{
    colorlinks,
    linkcolor={red!80!black},
    citecolor={blue!100!black},
}
\colorlet{linkequation}{blue}
\usepackage{authblk}

\usepackage[utf8]{inputenc} 
\usepackage[T1]{fontenc}    
\usepackage{url}            
\usepackage{booktabs}       
\usepackage{amsfonts}       
\usepackage{nicefrac}       
\usepackage{microtype}      
\usepackage{xcolor}         

\usepackage{natbib}
\usepackage{graphicx}
\usepackage{subcaption}
\usepackage{float}

\usepackage{amsmath}
\usepackage{amssymb}
\usepackage{mathtools}
\usepackage{amsthm}
\usepackage{enumitem}
\usepackage{tabularx}
\usepackage{bm}
\usepackage{tcolorbox}
\theoremstyle{plain}
\usepackage{wrapfig}
\usepackage{etoc}

\newtheorem{theorem}{Theorem}[section]
\newtheorem{proposition}[theorem]{Proposition}
\newtheorem{lemma}[theorem]{Lemma}
\newtheorem{corollary}[theorem]{Corollary}
\theoremstyle{definition}

\theoremstyle{remark}
\newtheorem{remark}[theorem]{Remark}

\tcbuselibrary{theorems}
\newtcbtheorem[number within=section]{mainresult}{Main Result}%
{colback=blue!5!white,colframe=blue!75!black,
 fonttitle=\bfseries, boxed title style={size=small}}%
{mr}
\usepackage[dvipsnames]{xcolor}

\usepackage{soul}

\newcommand{\JSCheck}[1]{#1}

\title{Asymmetric Scaling Laws from Sparse Features}

\author{%
  John Sous\thanks{Author to whom correspondence should be addressed.}\\
   Department of Applied Physics, Yale University, New Haven, Connecticut 06511, USA \& 
Energy Sciences Institute, Yale University, West Haven, Connecticut 06516, USA\\
   \texttt{john.sous@yale.edu} \\

     \and
     
  Michael Winer\thanks{Present affiliation: Alignment Research Center.}\\
  Institute for Advanced Study, Princeton, NJ 08540, USA\\
  \texttt{mikewins@ias.edu } \\
}

\date{\today}

\allowdisplaybreaks

\begin{document}

\maketitle

\vspace{-6mm}

\begin{abstract}
We introduce a model for neural scaling laws under sparse activations. In the model, test loss is often dominated by rare coordinates that are never observed in the training input. This mechanism induces a novel bottleneck absent from dense models. We derive the asymptotic \JSCheck{population} loss in both the underparameterized and overparameterized regimes, and show that the loss exhibits a double-descent peak near the interpolation threshold---where the number of parameters is just sufficient to fit the training data---resulting in a loss curve governed by two distinct scaling exponents---one for the overparameterized regime and one for the underparameterized regime---with a gap determined by the degree of sparsity. Additionally, we derive a compute-optimal frontier that favors increasing dataset size over model capacity under fixed compute budgets. \JSCheck{We also analyze gradient-descent dynamics and identify a scaling law for the probability that fixed-step gradient descent becomes unstable.} \JSCheck{We further show that the sparsity-induced effect persists under nonlinear activations.} Experiments validating the theory can be found at \href{https://anonymous.4open.science/r/sparse-scaling-neurips2026-3CDB/SparseScaling1.ipynb}{SparseScaling}.
\end{abstract}

\vspace{-2mm}

\section{Introduction}
\label{intro}

Scaling laws describe predictable relationships between model performance and key resource variables: model size (number of parameters) $N$, dataset size (number of training examples or tokens) $D$, and compute budget $C$. These empirical laws provide crucial guidance for efficiently allocating resources when training large machine learning models, particularly large language models (LLMs). Foundational work by Kaplan et al. showed that test loss $\ell$ decreases as a power law with respect to $N$ and $D$ when scaling one while holding the other fixed \citep{Kaplan2020ScalingLaws}. Building on this, Hoffmann et al. showed that under the constraint of a fixed total compute budget $C \propto N D$, both $N$ and $D$ should be scaled jointly in proportion to minimize $\ell$ \citep{Hoffmann2022Chinchilla}. This leads to a compute-optimal scaling law that offers a principled prescription for balancing model and data scale in large-scale training.

Despite intense research, a comprehensive theoretical understanding of scaling laws remains elusive. \JSCheck{One approach to this challenge models the data as Gaussian with power-law covariance, applies a random embedding into a lower-dimensional representation, and uses linear regression to characterize the resulting scaling laws~\citep{pmlr-v119-bordelon20a,Spigler_2020,Bahri2021ExplainingScaling,Maloney2022SolvableScaling}.} This framework~\citep{Maloney2022SolvableScaling} has been shown, both theoretically and empirically, to produce the scaling law $\ell \propto \left(\frac{1}{N} + \frac{1}{D}\right)^{\alpha}$ and the ``Chinchilla" compute-optimal allocation in which, under a compute budget $C \propto N D$, the optimal choice scales as $N^*(C)\asymp D^*(C)\asymp C^{1/2}$.\footnote{This compute-allocation law arises in this framework both in the Bayes-optimal~\citep{Bahri2021ExplainingScaling, Maloney2022SolvableScaling} and one-pass SGD~\citep{Bordelon2024DynamicalScaling,Paquette2024Phases} settings. In the Bayes-optimal case, this scaling can be found  by analytically extremizing the loss under the compute constraint. Results based on random matrix theory support the same scaling in one-pass stochastic gradient descent (SGD) dynamics.}

\JSCheck{An important aspect not determined by a framework that yields a compute-optimal exponent close to $\tfrac{1}{2}$ is whether the power-law exponent governing loss decay with model size, $\alpha_N$, matches that governing decay with data size, $\alpha_D$, as empirical studies suggest need not be the case~\citep{Hoffmann2022Chinchilla}.}\footnote{\JSCheck{Frameworks that predict $\alpha_N = \alpha_D = \alpha$ necessarily imply $N^*(C)\asymp D^*(C)\asymp C^{1/2}$ under the fixed-compute constraint $C \propto ND$. However, a small asymmetry between $\alpha_N$ and $\alpha_D$ can still yield compute-optimal exponents close to $\tfrac{1}{2}$. Thus, near-square-root compute-optimal scaling does not by itself imply symmetric scaling.}} \JSCheck{In this work, we propose a simple theoretical model that reproduces the observed scaling asymmetry and naturally explains the emergence of two distinct exponents in the  population loss.}\footnote{\JSCheck{\citet{bordelon2025featurelearning} also report scaling asymmetry, but there the relevant regimes are controlled by task difficulty rather than by the sparse-feature mechanism studied here.}}  \JSCheck{
Our model is a random embedding followed by a linear readout, with a sparse input activation structure such that only a subset of coordinates of \(\mathbf{x}\) is active. This setup is motivated by the idea that the data-generating process excites only a sparse subset of high-dimensional feature directions. The resulting loss displays distinct scaling with $N$ and $D$. We derive an asymptotic scaling law in which the difference between $\alpha_N$ and $\alpha_D$ arises continuously as the sparsity level is varied.}\footnote{We note that in training scenarios such as few-pass SGD, the underlying sparsity may be partially or entirely obscured. This disparity suggests a novel phase transition as the number of epochs increases.}

Our main contributions are as follows.
\begin{itemize}[leftmargin=*,nosep]
    \item We introduce the notion of scaling laws for sparse activations by introducing a model that captures the impact of sparsity on optimization and resulting scaling behavior.

    \item We derive the \JSCheck{population loss} for the sparse random feature model, yielding a two-exponent scaling law with an intrinsic asymmetry between the underparameterized and overparameterized regimes and a double-descent peak near the interpolation threshold~\citep{geiger2019jamming,belkin2019reconciling,belkin2020twomodels,hastie2022surprises,bartlett2020benign,mei2022rf,nakkiran2020deep}.

    \item We derive the compute-optimal frontier under a fixed compute budget, and show that increasing sparsity improves compute-efficiency  while shifting the optimal allocation toward larger datasets.

    \item We analyze the training dynamics and convergence properties of the loss during optimization, deriving a scaling law for the failure of fixed-step gradient descent (GD).

    \item \JSCheck{We experimentally verify in a  nonlinear two-layer network that scaling asymmetry continues to arise from sparsity rather than from nonlinearity.}
\end{itemize}

\vspace{-2mm}

\paragraph{Related Work.} Empirical scaling laws show power-law loss improvement with model size and data, and motivate compute-optimal training rules \citep{Kaplan2020ScalingLaws,Hoffmann2022Chinchilla}. On the theory side, \emph{solvable} models based on random features, kernels, and high-dimensional regression derive closed-form or deterministic-equivalent scaling predictions \citep{Bahri2021ExplainingScaling, Maloney2022SolvableScaling,pmlr-v119-bordelon20a,Spigler_2020, Defilippis2024DimensionFreeDE,Lin2024LinRegScaling}. Complementary \emph{dynamical} models study protocol- and time-dependent scaling (\textit{e.g.}, one-pass vs.\ multi-pass training) and can exhibit multiple scaling phases \citep{Bordelon2024DynamicalScaling, Paquette2024Phases, bordelon2025featurelearning}. In contrast, we focus on sparse feature activation, and show it can \emph{intrinsically} yield different exponents in the model-limited versus data-limited regimes even in Bayes-optimal learning. \JSCheck{We thus isolate rare feature coverage, rather than optimizer dynamics or feature learning, as a distinct source of scaling asymmetry.} See Appendix~\ref{appendix:AddRWork} for more discussion of related works.

\section{Statement of Problem}
In this section, we motivate the model and outline the main goals of our analysis.

\subsection{Learning under Sparse High-Dimensional Data}\label{Section:Ourmodel}
We first  specify the data-generation process that yields a sparse power-law structure in the inputs, then describe the random feature model used to learn from these high-dimensional inputs.

\subsubsection{Data Generation Process: Bernoulli-Random Activations with Power-Law Covariance}

We consider input data \(\mathbf{X} = [\,\mathbf{x}_1,\dots,\mathbf{x}_D\,] \in \mathbb{R}^{M\times D}\), with each \(\mathbf{x}_d  \in \mathbb{R}^{M}\), where $D$ is the number of training examples (data size) and $M$ is the data dimension. We henceforth drop the $d$ subscript and simply refer to a single training example as $\mathbf{x}$.  We consider a structured data-generation process in which each input coordinate of \(\mathbf{x}  \in \mathbb{R}^{M}\) is randomly activated according to a heavy-tailed sparsity pattern. Specifically, for \(j=1,2,\dots, M\), the \(j\)th coordinate of \(\mathbf{x}\) is drawn as
\begin{equation}
\mathbb{P}(x_j = 0) = 1 - j^{-\alpha_1 - 1}, \qquad\mathbb{P}\!\left(x_j = \pm j^{-(\alpha_2 + 1)/2}\right) 
= \tfrac{1}{2}\, j^{-\alpha_1 - 1}.
\label{eq:sparse-active-prob}
\end{equation}
Hence $\mathbb{E}[x_j] = 0$ and
$\operatorname{Var}(x_j) = j^{-\alpha_1 - \alpha_2 - 2}$.
Let \( p_j := \mathbb{P}(x_j \neq 0) = j^{-\alpha_1 - 1} \). We require \(\alpha_1 \geq -1\) to ensure that \(p_j \leq 1\). Additionally, to ensure that the \JSCheck{target} variance is finite, the sum over the coordinate-wise variances \(\sum_{j=1}^\infty j^{-\alpha_1 - \alpha_2 - 2}\) must converge, which holds iff \(\alpha_1 + \alpha_2 + 1 > 0\). 

This formulation introduces sparsity through Bernoulli input activations: most coordinates of \(\mathbf{x}\) are zero, while the few active ones follow a heavy-tailed scaling controlled by \(\alpha_1\) and \(\alpha_2\).

\subsubsection{Random Feature Model}
We study a linear model built from \( N \ll M \)  features randomly embedded from these inputs\JSCheck{, with $M\gg\max\{N,D\}$ so that truncation does not affect the power-law tail asymptotics.}  Each input $\mathbf{x}\in\mathbb{R}^M$ is mapped to features $\boldsymbol{\phi}$ via a fixed random embedding $\mathbf{u} \in\mathbb{R}^{N\times M}$, with  $u_{ij} \sim \mathcal{N}(0, 1/N)$ and its $N\times M$ elements picked i.i.d., $\boldsymbol{\phi} = \mathbf{u} \mathbf{x}$, and the (per-example) prediction is $\hat y(\mathbf{x})=\boldsymbol{\theta}^\top \boldsymbol{\phi}=\boldsymbol{\theta}^\top \mathbf{u} \mathbf{x}$ (and the batched prediction is $\hat{\mathbf{y}} =\boldsymbol{\theta}^\top \boldsymbol{\Phi}=\boldsymbol{\theta}^\top \mathbf{u} \mathbf{X}$), with trainable $\boldsymbol{\theta}\in\mathbb{R}^N$.\footnote{\JSCheck{Our model is, in effect, a sketched linear regression. It is random-feature-like in the sense of a frozen random map followed by a trained linear readout, but it differs from canonical random-features models, \textit{e.g.}~\cite{hastie2022surprises,gerace2020generalisation,goldt2020hiddenmanifold}, which study nonlinear feature maps in high-dimensional regimes. Extending those frameworks to the sparse, strongly non-Gaussian/heavy-tailed designs studied here would require going beyond the Gaussian-equivalence tools used in  the nonlinear random-features literature~\cite{hu2022universality,mei2022rf,GoldtEtAl2022GaussianEquivalence}.}} The target function is $y(\mathbf{x})=\mathbf{w}^\top \mathbf{x}=\sum_{j=1}^M w_j x_j$, for random $\mathbf{w}\in\mathbb{R}^M$ with $w_j$ i.i.d. and $\operatorname{Var}(w_j)=1$.  The model parameters and their specifications are summarized in~\tableautorefname~\ref{tab:matrixTable}.

\begin{table}[h]
\caption{Model parameters: Shapes and elementwise scaling.}
\label{tab:matrixTable}
\small
\setlength{\tabcolsep}{4pt}
\renewcommand{\arraystretch}{1.15}
\begin{tabularx}{\columnwidth}{
    l
    >{\hsize=1.25\hsize}X
    >{\hsize=0.4\hsize\centering\arraybackslash}X
    >{\hsize=1.3\hsize}X
}
\toprule
Parameter & Description & Shape & Scaling \\
\midrule
$\mathbf{X}$ & Training data (columns are examples) & $M\times D$ & Per-coordinate std dev. $j^{-(\alpha_1+\alpha_2+2)/2}$ \\
$\mathbf{u}$ & Fixed random embedding & $N\times M$ & $u_{ij}\sim\mathcal{N}(0,1/N)$, i.i.d. \\
$\boldsymbol{\phi}$ & Features for a single input & $N\times 1$ &  \\
$\boldsymbol{\Phi}$ & Features for the full dataset & $N\times D$ &  \\
$\boldsymbol{\theta}$ & Trainable readout & $N\times 1$ & $O(1)$ \\
$\mathbf{w}$ & Target linear functional & $M\times 1$ & $O(1)$ \\
$y$ & Per-example target function value  & scalar & \\
$\mathbf{y}$ & Per-batch target function & $1\times D$ & $O(1)$ \\
\bottomrule
\end{tabularx}
\vspace{-2mm}
\end{table}

This generative model allows us to examine how sparsity and heavy-tailed structure in the data influence performance. As we show below, these structural properties induce a loss scaling law with two distinct exponents, $\alpha_N \neq \alpha_D$, a sparsity-dependent compute-optimal frontier with an exponent $\alpha_C$ that shifts the optimal allocation toward data, and a regime in which GD fails to converge. The resulting regimes are summarized in Figure~\ref{fig:phase-diagram}. \JSCheck{This scaling asymmetry also persists in experiments with nonlinear MLPs.}

\begin{wrapfigure}[22]{L}{0.525\textwidth}
\vspace{-0.45cm}
\centering
\includegraphics[width=0.5\textwidth]{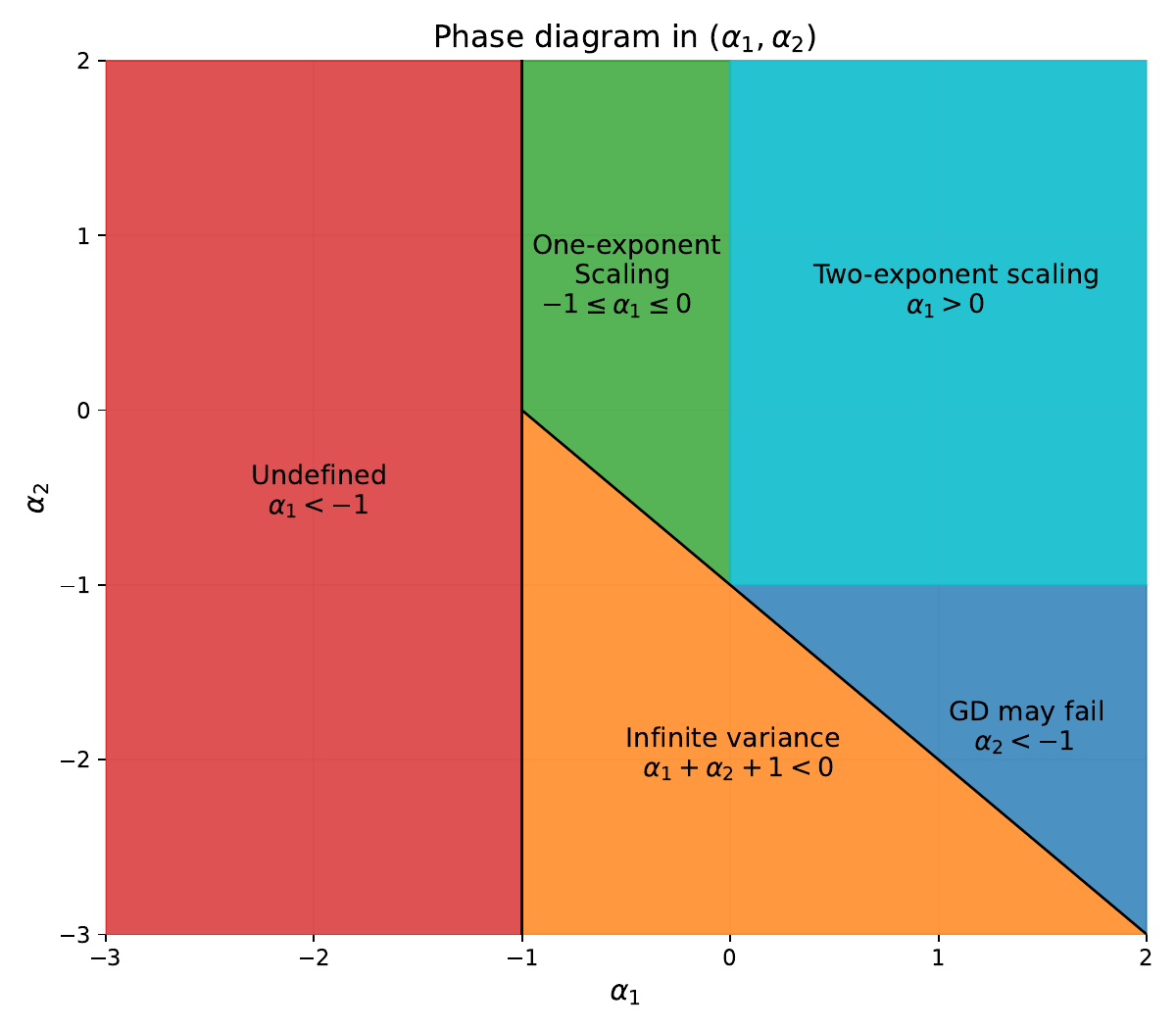}
\vspace{-2mm}
\caption{
\textbf{Phase diagram in $(\alpha_1,\alpha_2)$.} Solid black lines mark the boundary beyond which the model is well defined. Within it we identify three regimes: symmetric one-exponent scaling, asymmetric two-exponent scaling, and a GD-failure regime with its own scaling law.
}
\label{fig:phase-diagram}
\end{wrapfigure}

\vspace{-5mm}
\paragraph{Interpretation: High-Dimensional Representations with Rare but Informative Features.}  Sparse activations of \(\mathbf{x}\) can be understood as rare but highly informative features that appear only in a small fraction of examples yet carry disproportionate predictive value, \textit{e.g.}, an extremely infrequent keyword such as ``pheochromocytoma,'', rare financial shocks such as the 2008 crisis, or an uncommon diagnostic biomarker.

The random feature model can be motivated by interpreting the high-dimensional vector \(\mathbf{x}\) as a nonlinear expansion of a much lower-dimensional latent variable \(\tilde{\mathbf{x}}\)~\citep{rahimi2007randomfeatures}. The components of \(\mathbf{x}\) could correspond to all polynomial or Fourier basis functions of \(\tilde{\mathbf{x}}\), while the features \(\boldsymbol{\phi}=\mathbf{u}\mathbf{x}\) represent the smaller subset of directions actually captured by our model. Although the mapping through \(\mathbf{u}\) is entirely linear, introducing randomness, it can be interpreted as the first layer of a two-layer network (with no nonlinearity) whose weights are frozen at random initialization~\citep{cho2009kernel}. In this view, the model explores how much predictive structure can be recovered when learning is restricted to a fixed, randomly oriented subspace of the high-dimensional representation.

Thus, this random feature model with sparse activations reveals how rare, high-impact signals shape optimization and give rise to distinct scaling laws, highlighting how sparsity fundamentally alters achievable performance.

\subsection{Main Goals}

Our goal is to characterize the trained-to-completion loss for the random sparse feature model described above. The population loss of the trained estimator is (Appendix~\ref{appendix:computation})
\begin{equation}
\ell_{\mathrm{Bayes}} = \mathbb{E}_{\mathbf{x}}\left[(y - \hat{y}(\mathbf{x}))^2\right] = \mathbb{E}_{\mathbf{x}}\!\left[ \big( \mathbf{w}^\top \mathbf{x}
- \boldsymbol{\theta}^{\star\!\top}\mathbf{u}\mathbf{x} \big)^2 \right],
\label{eq:bayes-loss}
\end{equation}
\JSCheck{where \(\boldsymbol{\theta}^\star\) denotes the minimum-\(\ell_2\)-norm empirical-risk minimizer, equivalently the solution reached by GD from zero initialization when the empirical loss is trained to completion. Note that
\(\hat{y}(\mathbf{x}) = \boldsymbol{\theta}^{\star\top}\mathbf{u}\mathbf{x}\) corresponds to a linear function of
\(\mathbf{x}\) whose weight vector lies in the row space of \(\mathbf{u}\). Consequently, when
\(\mathbf{w} \notin \operatorname{rowspan}(\mathbf{u})\), this feature map cannot recover
\(y = \mathbf{w}^\top \mathbf{x}\) exactly, and a nonzero approximation error is unavoidable.}

\JSCheck{We aim to derive the scaling laws and compute-optimal frontier for this Bayes-optimal loss, emphasizing the dependence on the data distribution and its effect on optimization. In particular, we show that a two-exponent loss scaling law and a sparsity-dependent compute-optimal exponent can arise purely from the structure of the input distribution, in which \(\mathbf{x}\) contains sparse but informative coordinates. Our objective is to identify exponents \(\alpha_N\), \(\alpha_D\), and \(\alpha_C\) such that $\ell_{\mathrm{Bayes}} \;\asymp\; N^{-\alpha_N} + D^{-\alpha_D}$, and the compute-optimal frontier satisfies $\ell_{\mathrm{Bayes}}^\star \;\asymp\; C^{-\alpha_C}$,
where \(\ell_{\mathrm{Bayes}}^\star\) is the loss achieved by compute-optimal allocation. We also study GD in this setting, examining how optimization dynamics interact with the random feature map and the sparse data distribution, and whether it attains the Bayes rates or instead exhibits additional computational constraints, which we dub a scaling law of failure of GD.  All proofs are deferred to the appendices.
}

\section{\JSCheck{Scaling Asymptotics of Sparse Random Features}}\label{Sec:scaling1}

Our strategy is to analyze the scaling of the Bayes-optimal loss,  Eq.~\eqref{eq:bayes-loss}, for our model of sparse random features (specified in Section~\ref{Section:Ourmodel}) in the
\emph{underparameterized} regime, where $N$
is much smaller than $D$ and in
the \emph{overparameterized} regime, where \(N\) vastly exceeds \(D\). We will show
that these regimes are governed by two distinct power laws, with different
exponents, demonstrating that the loss is asymmetric and thus follows a two-exponent scaling law. This symmetry is broken precisely due to the sparsity of the input activation pattern.  See Appendix~\ref{App:Scaling} for details. 

\subsection{Loss Scaling from Unmodeled Features}

The random feature model reduces to linear regression in the effective weight vector \(\hat{\mathbf{w}} = \mathbf{u}^\top \boldsymbol{\theta}\). The population mean-squared error can be expressed as $\ell(\hat{\mathbf{w}}) = \mathbb{E}_{\mathbf{x}}\!\left[(\mathbf{w}^\top\mathbf{x} - \hat{\mathbf{w}}^\top\mathbf{x})^2\right]
= (\mathbf{w}-\hat{\mathbf{w}})^\top \boldsymbol{\Sigma}_x (\mathbf{w}-\hat{\mathbf{w}})$, where \(\boldsymbol{\Sigma}_x = \mathbb{E}[\mathbf{x}\mathbf{x}^\top]\) is diagonal
since the input coordinates are independent. \JSCheck{To expose the scaling mechanism, consider a predictor that matches the first \(k\) coordinates of \(\mathbf{w}\) and sets all remaining coordinates to zero: $\hat{w}_j = w_j \mathbf{1}_{\{j \le k\}}$}. Then \(\mathbf{w}-\hat{\mathbf{w}}\) has entries \(0\) for \(j \le k\) and \(w_j\)
for \(j>k\), and substituting into the expression for \(\ell(\hat{\mathbf{w}})\) yields
\(\ell(\hat{\mathbf{w}})
= \sum_{j>k} w_j^2\, \operatorname{Var}(x_j)\). Under our assumptions \(w_j^2 = O(1)\), the population loss is therefore governed, up to constants, by the \emph{unmodeled input variance} $\sum_{j>k} \operatorname{Var}(x_j)$. \JSCheck{This calculation suggests  that the asymptotic behavior of the loss is governed by the tail of the variance spectrum of the input distribution:}
\begin{equation}
\ell_{\mathrm{Bayes}} \;\asymp\; \sum_{j>k} \operatorname{Var}(x_j).
\label{eq:bayes-tail}
\end{equation}
\JSCheck{The argument above is only an intuition-building asymptotic device to make the mechanism transparent. Appendix~\ref{App:Scaling}, particularly Sections~\ref{prop:underparam} and~\ref{App:rig-proof}, provide a rigorous proof.}

\subsection{Prior Results: Scaling Law for Uniformly Activated Data with Power-Law Covariance}

Before analyzing our model with sparse activations, we first review the scaling laws established for random feature models with fully active (non-sparse) input coordinates~\citep{Maloney2022SolvableScaling}. These results form the baseline against which the effects of sparsity will be contrasted.

\vspace{-2mm}

\paragraph{Uniformly Activated Data with Power-Law Covariance.}
In the original formulation of~\citet{Maloney2022SolvableScaling}, every input coordinate is active
and is modeled as an independent Gaussian variable with variance decaying as a
power law: $x_j \sim \mathcal{N}(0,\, j^{-\alpha - 1}), \qquad j = 1,\dots,M,\quad \alpha>0$. This matches our model’s fully activated covariance structure under the choice $\alpha_1=-1$ and $\alpha_2=\alpha$.

\vspace{-2mm}

\paragraph{Symmetry between Underparameterized and Overparameterized Regimes.}
For this non-sparse data structure, the optimal-loss scaling law was analyzed in detail in~\citep{Maloney2022SolvableScaling}. We produce a simple \JSCheck{intuitive} argument showing why, in this
setting, a \emph{single} exponent governs the asymptotic behavior of the loss in both the under- and overparameterized regimes, yielding a symmetric one-exponent law $\ell_{\mathrm{Bayes}} \;\propto\; \left(\frac{1}{N} + \frac{1}{D}\right)^{\alpha}$ with $\alpha_N = \alpha_D = \alpha.$  As explained above, the random feature model reduces to linear regression in the effective weight
vector. Thus, if a learned predictor resolves the leading \(k \ll M\) coefficients of \(\mathbf{w}\), the residual (unmodeled) variance is $\sum_{j>k} j^{-\alpha-1} \;\asymp\; \int_k^\infty j^{-\alpha-1}\,dj \;=\; \frac{1}{\alpha}\,k^{-\alpha}$. With finitely many parameters \(N\) and finitely many samples \(D\), the model can identify at most \(O(\min\{N,D\})\) coefficients, since this minimum represents the true bottleneck. This yields the scaling
$\ell_{\mathrm{Bayes}} \;\sim\; \min\{N,D\}^{-\alpha}$,
so that the exponents satisfy
\(\alpha_N = \alpha_D = \alpha\). In the sparse model of Section~\ref{Section:Ourmodel}, one still gets a symmetric one-exponent scaling law when $-1\le \alpha_1 \le 0$. In this regime, with high probability the low-index coordinates are activated sufficiently often across $D$ samples that sparsity is effectively subdominant. The scaling is therefore still controlled by $\min\{N,D\}$, as in the dense case, yielding $\alpha_N = 
\alpha_D = \alpha_1 + \alpha_2 + 1$ (Appendix~\ref{app:weak-sparsity}).

\subsection{Two-Exponent Scaling Law for Sparse Activations ($\alpha_1>0$)}

We now analyze our model of sparse activations for $\alpha_1>0$ (Appendix~\ref{appendix:TwoExpon}). We begin by analyzing the asymptotic behavior of the loss in the underparameterized regime, where we will show that sparsity is effectively hidden. We then turn to the overparameterized regime, where sparsity becomes consequential and leads to a distinct asymptotic loss behavior. We work in the joint limit $N,D\to\infty$, \JSCheck{and we consider asymptotic regimes corresponding to different relative scaling of $N$ and $D$.}\footnote{In the sparse activation model, the relevant comparison is ultimately between $N$ and the  number of input coordinates ever observed in the dataset, $K(D)$ (quantified below), which grows sublinearly in $D$ when $\alpha_1>0$.}

\subsubsection{Underparameterized Regime (\(N \ll D\)).}

In this regime,  the model fits only the leading high-variance coordinates of  $\mathbf{w}$, and the resulting scaling law coincides with that of the  non-sparse case.

\begin{proposition}[Underparameterized Scaling]
\label{prop:underparam}
Under the sparse activation model described in 
Section~\ref{Section:Ourmodel}, the Bayes-optimal loss in the 
underparameterized regime ($N \ll D$) satisfies
\begin{equation}
\ell_{\mathrm{Bayes}, N}
\;\sim\;
N^{-(\alpha_1 + \alpha_2 + 1)},
\quad
\text{so that }
\alpha_N = \alpha_1 + \alpha_2 + 1.
\label{eq:bayes-N-scaling-prop}
\end{equation}
\end{proposition}

\vspace{-0.5mm}
\noindent\textit{Interpretation.}
\JSCheck{The model captures only the leading high-variance coordinates of \(\mathbf{w}\). Since sparse coordinates are too rarely observed to influence estimation at this scale, they do not affect the loss. Consequently, the exponent $\alpha_N = \alpha_1 + \alpha_2 + 1$ coincides with the non-sparse case, so sparsity is effectively masked in the underparameterized regime.}

\subsubsection{Overparameterized Regime ($D \ll N$).}

\JSCheck{A key quantity in this regime is the number of input coordinates that are ever observed} (\textit{i.e.}, activated at least once). We therefore begin by quantifying the effective number of learnable coordinates as a function of $D$, which will set the correct scale separating the two asymptotic regimes.

\begin{lemma}[Learnable Coordinates in the Sparse Model when $\alpha_1>0$]
\label{lem:learnable-coordinates}
Under the sparse activation distribution
$p_j = \mathbb{P}(x_j \neq 0) = j^{-\alpha_1 - 1}$, the expected number of coordinates that are active in at least one of the $D$ training samples is asymptotic to the scale
\begin{equation}
K(D)
=
\Gamma\!\left(1 - \frac{1}{\alpha_1 + 1}\right)
D^{\frac{1}{\alpha_1 + 1}}, \footnote{When $\alpha_1>0$, $K(D)\asymp D^{1/(\alpha_1+1)}=o(D)$, \textit{i.e.} only a sublinear number of coordinates are ever observed. Interpolation is still possible since the data lie in a $K(D)$-dimensional subspace, which a linear predictor in $\mathbb{R}^N$ can fit. Unobserved coordinates receive no signal and are typically suppressed by mild regularization.
}
\label{eq:K-of-D}
\end{equation}
where $\Gamma(\cdot)$ is the Gamma function. Moreover, $K(D)$ \JSCheck{sets the scale of the} number of coefficients of $\mathbf{w}$ that can be estimated from data: a coordinate that is never activated carries no information about its associated weight, and hence cannot contribute to the learned predictor.
\end{lemma}

\begin{theorem}[Overparameterized Scaling]
\label{thm:over-scaling}
Under the sparse activation model with $\alpha_1>0$ and coordinate variances
$\operatorname{Var}(x_j) = j^{-\alpha_1 - \alpha_2 - 2}$, the Bayes-optimal loss
in the overparameterized regime satisfies
\begin{equation}
\ell_{\mathrm{Bayes},D}
\;\asymp\;
D^{-\frac{\alpha_1 + \alpha_2 + 1}{\alpha_1 + 1}},
\quad
\text{so that }
\alpha_D = \frac{\alpha_1 + \alpha_2 + 1}{\alpha_1 + 1}.
\label{eq:bayes-loss-over}
\end{equation}
\end{theorem}
A rigorous proof of Theorem~\ref{thm:over-scaling} is given in Appendix~\ref{App:rig-proof} via a continuous argument.

\begin{remark}[Intuition]
A simple heuristic recovers the same sample-limited exponent as in 
Theorem~\ref{thm:over-scaling}. Coordinate $j$ is active in a single sample  with probability $p_j = j^{-\alpha_1 - 1}$, so across $D$ samples the expected  number of activations is $D p_j$. A coordinate is observed often enough to be estimable once $D p_j \gtrsim 1$, which yields the cutoff $j \;\lesssim\; j_D \;\equiv\; D^{1/(\alpha_1+1)}$. Thus the number of learnable coordinates scales as  $K(D) \asymp j_D$, in agreement with  Lemma~\ref{lem:learnable-coordinates}. Coordinates with $j > j_D$ are effectively unobserved. Their contribution to the test error is the unmodeled variance $\sum_{j>j_D} j^{-\alpha_1 - \alpha_2 - 2} \;\asymp\; j_D^{-(\alpha_1+\alpha_2+1)} \;=\; D^{-\frac{\alpha_1+\alpha_2+1}{\alpha_1+1}}$. \JSCheck{Consequently, \(\ell_{\mathrm{Bayes},D} \asymp D^{-\frac{\alpha_1+\alpha_2+1}{\alpha_1+1}}\),} with $\alpha_D = \frac{\alpha_1+\alpha_2+1}{\alpha_1+1}$, matching the rigorous sample-limited exponent in  Theorem~\ref{thm:over-scaling}.
\end{remark}

\begin{corollary}[Asymmetry of Exponents]
\label{cor:asymmetry}
For all \(\alpha_1 > 0\) and \(\alpha_1+\alpha_2+1>0\), the exponents in the underparameterized and overparameterized regimes satisfy
\begin{equation}
\alpha_D
=
\frac{\alpha_1 + \alpha_2 + 1}{\alpha_1 + 1}
\;<\;
\alpha_1 + \alpha_2 + 1
=
\alpha_N.
\end{equation}
Thus,  for $\alpha_1 > 0$, the sparse model  exhibits an intrinsic asymmetry between the sample-limited and parameter-limited regimes, giving rise to the
two-exponent scaling law.
\end{corollary}

Together, Proposition~\ref{prop:underparam} and Theorem~\ref{thm:over-scaling} yield:
\begin{center}
\begin{tcolorbox}[
    colback=white,
    colframe=black,
    boxrule=0.5pt,
    arc=3pt,
    left=2pt,
    right=2pt,
    top=-2pt,
    bottom=4pt,
    width=0.7\columnwidth,
    before skip=0pt,
    after skip=0pt
]
\begin{equation}
\ell_{\mathrm{Bayes}}(N, D)
\;\asymp\;
N^{-(\alpha_1 + \alpha_2 + 1)}
\;+\;
D^{-(\alpha_1 + \alpha_2 + 1)/(\alpha_1 + 1)}.
\label{eq:two-exponent-summary}
\end{equation}
\end{tcolorbox}
\end{center}
\vspace{-1mm}
\noindent
\JSCheck{This expression makes the asymmetry explicit: model-limited error decays sharply, whereas data-limited error decays slowly. Since the mechanism depends on sparsity and variance profiles, we expect it to extend beyond Bernoulli masks to broader sparse distributions.}

\section{\JSCheck{Empirical Study of the Scaling Laws}}
\label{sec:empirical-scaling}

We now empirically validate the theoretical predictions of the previous section by evaluating the \JSCheck{population loss} over a wide range of $D$ and $N$ in the sparse random-feature model of Section~\ref{Section:Ourmodel}.  Computational details are provided in Appendices~\ref{appendix:computation} and~\ref{app:training}.

\subsection{Scaling Collapse and Universality}
\label{sec:scaling-collapse}

\begin{wrapfigure}[23]{R}{0.525\textwidth}
\vspace{-0.8cm}
\centering
\includegraphics[width=0.5\textwidth]{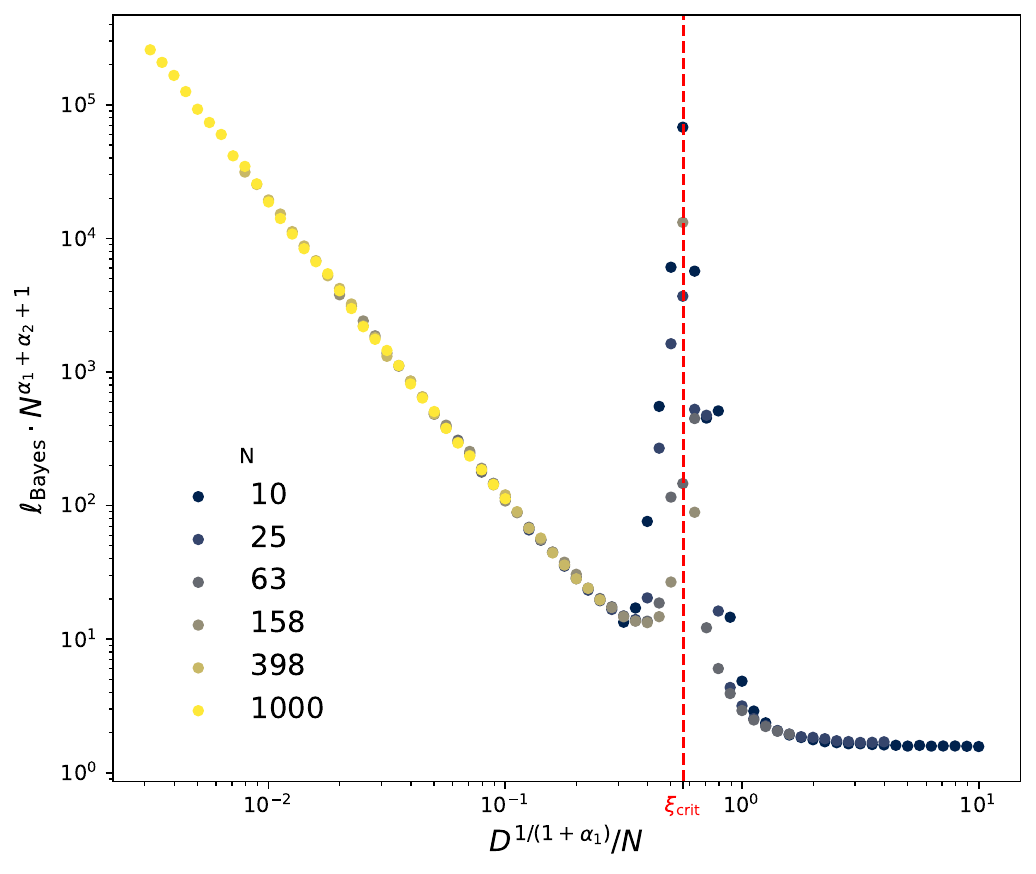}
\vspace{-2mm}
    \caption{
    \textbf{Scaling collapse and double descent.} We plot the rescaled loss \( \text{$\ell_\mathrm{Bayes}$} \cdot N^{\alpha_N} \) as a function of \( (D^{\alpha_D}/N^{\alpha_N})^{1/\alpha_N} \), across a wide range of values for \( D \), \( N \), for fixed $(\alpha_1 = 1.0, \alpha_2 = 0.3)$. All curves collapse onto a single universal function \( \mathcal{S}_{\alpha_1, \alpha_2} \), verifying a universal scale-invariant structure consistent with the predicted scaling law. A sharp peak emerges near $\xi_\mathrm{crit}$ consistent with the double descent phenomenon at the interpolation threshold.
    }
\label{fig:scaling-collapse}
\end{wrapfigure}

In many models, including ours and others such as the Chinchilla scaling law~\citep{Hoffmann2022Chinchilla}, the test loss exhibits distinct scaling in two regimes: \( \ell \sim D^{-\alpha_D} \) in the data-limited regime and \( \ell \sim N^{-\alpha_N} \) in the parameter-limited regime. A natural question is how to interpolate between these two limits. A common ansatz is that the loss takes the form $\ell = \bar{L}(D^{-\alpha_D}, N^{-\alpha_N})$
for some \emph{scale-invariant} function \( \bar{L} \) satisfying \( \bar{L}(cL_1, cL_2) = c \bar{L}(L_1, L_2) \). Simple choices include \( \bar{L}(L_1, L_2) = \max(L_1, L_2) \) or \( L_1 + L_2 \), but more refined analyses~\citep{Maloney2022SolvableScaling,zhang2024largenfieldtheory} show that \( \bar{L} \) often exhibits a \emph{peak} near \( L_1 \approx L_2 \), a phenomenon known as \emph{double descent}~\citep{belkin2019reconciling, hastie2022surprises}.\footnote{\JSCheck{The double-descent peak is a feature of benign interpolation under ridgeless fitting, rather than of the scaling-law picture itself. Regularization or early stopping smooths it out~\citep{hastie2022surprises}.}} Using the scale invariance of \( \bar{L} \), we may write $\ell \cdot N^{\alpha_N} = \bar{L}\left( \frac{D^{-\alpha_D}}{N^{-\alpha_N}}, 1 \right)$,
so the rescaled loss \( \ell \cdot N^{\alpha_N} \) becomes a universal function of the compute ratio $\Xi := \frac{D^{\alpha_D}}{N^{\alpha_N}}$. \JSCheck{We therefore expect the rescaled loss curves from different \( (D, N) \)  to collapse onto a single curve \( \mathcal{S}_{\alpha_1, \alpha_2}(\Xi) \):  a dimensionless scaling function of the rescaled compute ratio.}

\vspace{-2mm}

\paragraph{Empirical Scaling Collapse.} In Figure~\ref{fig:scaling-collapse}, we find that the full loss curve across a wide range of \( (D, N) \) pairs collapses onto a single universal function after appropriate rescaling: $\ell_{\mathrm{Bayes}} \cdot N^{\alpha_1 + \alpha_2 + 1} = 
\mathcal{S}_{\alpha_1, \alpha_2}\left( \frac{1}{N} D^{\frac{1}{\alpha_1 + 1}} \right)$, where $\mathcal{S}(u)=\bar{L}(u^{-\alpha_N},1)$, so that
\JSCheck{$\bar{L}(\frac{D^{-\alpha_D}}{N^{-\alpha_N}},1)
=\mathcal{S}_{\alpha_1,\alpha_2}\!\left(\xi\right)$, with $\xi := \left(\frac{D^{\alpha_D}}{N^{\alpha_N}}\right)^{1/\alpha_N}
=\frac{D^{1/(\alpha_1+1)}}{N}$}. This implies that knowing the loss curve for one setting of \( (D, N) \) suffices to predict its shape at other scales via a simple rescaling transformation. \JSCheck{This collapse supports the two-regime scaling theory and recovers the predicted asymptotic behaviors: in the underparameterized regime, \( \ell \sim N^{-\alpha_N} \) with \( \alpha_N=\alpha_1+\alpha_2+1 \), while in the overparameterized regime, \( \ell \sim D^{-\alpha_D} \) with \( \alpha_D=(\alpha_1+\alpha_2+1)/(\alpha_1+1) \).}

\vspace{-2mm}

\paragraph{Double Descent Peak.} \JSCheck{In the underparameterized limit, the loss scales as \( \ell \sim N^{-(\alpha_1+\alpha_2+1)} \), while in the overparameterized limit it scales as \( \ell \sim D^{-(\alpha_1+\alpha_2+1)/(\alpha_1+1)} \).} This suggests a crossover at \( N \sim D^{1/(\alpha_1+1)} \), which coincides with the point at which the number of model parameters $N$ matches the number of activated (and hence learnable) coordinates \( K(D) \) from Lemma~\ref{lem:learnable-coordinates}. The shape of \( \mathcal{S}_{\alpha_1, \alpha_2} (\xi) \) in Figure~\ref{fig:scaling-collapse} reveals a prominent \emph{double descent} phenomenon near:
\begin{equation}
\xi_{\mathrm{crit}} = \frac{1}{\Gamma\left( \frac{\alpha_1}{1 + \alpha_1} \right)},
\end{equation}
corresponding to the  threshold where $K(D)$ matches $N$\JSCheck{, with location determined universally by the sparsity exponent \(\alpha_1\).}  

\section{Compute-Optimal Scaling Laws}
\label{sec:compute-optimal}
\JSCheck{We now analyze how to allocate data and model size under a fixed compute budget so as to minimize the loss. We adopt the compute model $C = N D \cdot \min\{N, D\}$ as a unified proxy for the cost of training to completion in our setting
; see Appendix~\ref{appendix:compute} for discussion. We then minimize $\ell(N,D)\sim N^{-\alpha_N}+D^{-\alpha_D}$ subject to fixed compute $C$ to derive the compute-optimal scaling laws.}

\begin{proposition}[Compute-Optimal Frontier]
\label{prop:compute-optimal}
Let the Bayes-optimal loss, for $\alpha_1>0$, scale as  $\ell(N,D) \sim N^{-\alpha_N} + D^{-\alpha_D}$, with $\alpha_N = \alpha_1 + \alpha_2 + 1$ and $\alpha_D = \frac{\alpha_1 + \alpha_2 + 1}{\alpha_1 + 1}$. 
Under a fixed compute budget $C = N D \cdot \min\{N, D\}$, the unique compute-optimal allocation lies in the underparameterized regime \( N < D \), and the optimal allocation and resulting loss scale as
\begin{equation}
\begin{split}
\ell^*(C) &\sim C^{-\alpha_C}, \\
N^*(C) &\sim C^{\alpha_D/(\alpha_N+2\alpha_D)} = C^{1/(\alpha_1+3)}, \qquad
D^*(C) \sim C^{1-2\alpha_D/(\alpha_N+2\alpha_D)} = C^{(\alpha_1+1)/(\alpha_1+3)},\\
\alpha_C &:= \frac{\alpha_N\alpha_D}{\alpha_N+2\alpha_D}
= \frac{\alpha_1+\alpha_2+1}{\alpha_1+3} > 0 .
\end{split}
\label{eq:compute-opt-ND}
\raisetag{4ex}
\end{equation}
\end{proposition}



\begin{remark}[Absence of Overparameterized Optimum]\label{Remark:ComputeOverParam}
\JSCheck{The compute-optimal solution always lies in the underparameterized regime: $N^*(C) < D^*(C)$ for all $C$. An overparameterized optimum $N>D$, with $C=ND^2$, yields $N^*(C)/D^*(C) \sim C^{(\alpha_D-\alpha_N)/(2\alpha_N+\alpha_D)} \ll 1$, contradicting the assumed regime $N>D$. Hence the only valid optimum satisfies $N<D$.}
\end{remark}

\begin{remark}[Impact of Sparsity on Optimal Scaling] 
\JSCheck{As $\alpha_1$ increases, individual samples become more sparse, carrying less information per example.  At the same time, increasing $\alpha_1$ also steepens the marginal variance spectrum $\operatorname{Var}(x_j)=j^{-\alpha_1-\alpha_2-2}$, so the residual tail loss decays more rapidly once coordinates are resolved. Consequently, for fixed \(\alpha_2<2\), the compute-optimal exponent $\alpha_C = \frac{\alpha_1 + \alpha_2 + 1}{\alpha_1 + 3}$ increases monotonically with $\alpha_1$, approaching $1$ as $\alpha_1 \to \infty$. As $\alpha_1$ increases, the optimal model size, $N^*(C)$, grows more slowly with compute, while the optimal dataset size, $D^*(C)$, grows more rapidly. As a result, the allocation shifts toward spending a larger fraction of compute on data rather than parameters.}
\end{remark}

\vspace{-3mm}

\paragraph{Empirical Validation.}
\JSCheck{To validate Proposition~\ref{prop:compute-optimal}, Figure~\ref{fig:compute-scaling} shows test loss versus total compute $C$ for different $N$. Each curve corresponds to a fixed $N$, while $C$ is varied by sweeping $D$. The dashed line shows the predicted compute-optimal frontier, $\ell^*(C)\sim C^{-\alpha_C}$. As $C$ increases, the empirical curves approach this frontier, confirming the predicted scaling.}

\begin{wrapfigure}[21]{R}{0.525\textwidth}
\vspace{-0.1cm}
\centering
\includegraphics[width=0.5\textwidth]{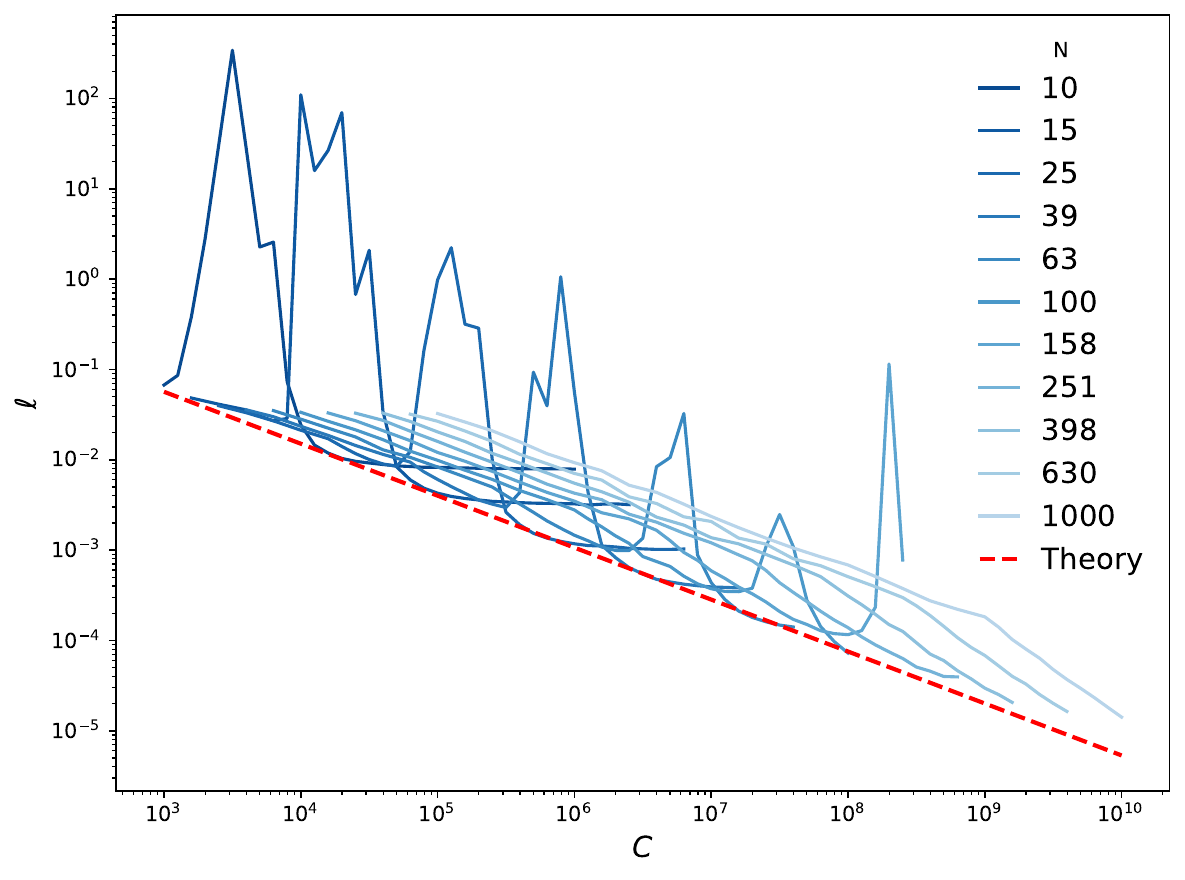}
\vspace{-2mm}
\caption{
\textbf{Empirical scaling of loss with compute.} Each curve corresponds to test loss $\ell$ versus total compute $C$ \JSCheck{for a fixed model size $N$, with $C$ varied by sweeping $D$}, for fixed $(\alpha_1=1.0,\alpha_2=0.3)$. The dashed line denotes the predicted compute-optimal scaling $\ell^*(C) \sim C^{-\alpha_C}$. The empirical loss converges toward this frontier at high compute, \JSCheck{just past the double descent spike for each curve} demonstrating agreement between theory and experiment.}
\label{fig:compute-scaling}
\end{wrapfigure}

\section{Gradient Descent Training Dynamics}
\label{sec:training_dynamics}

\JSCheck{We next study the training dynamics of the readout weights $\boldsymbol{\theta}$ under full-batch GD, a broadly meaningful computational model with some generality beyond the toy setting, on the empirical squared loss $\ell(\boldsymbol{\theta})=\frac{1}{2D}\left\|\mathbf{w}^\top\mathbf{X}-\boldsymbol{\theta}^\top\mathbf{u}\mathbf{X}\right\|_2^2$. With step size $\eta$, one GD step is $\Delta \boldsymbol{\theta}=\frac{\eta}{D}\,\mathbf{u}\mathbf{X}\mathbf{X}^\top\left(\mathbf{w}-\mathbf{u}^\top\boldsymbol{\theta}\right)$. It is convenient to rewrite the dynamics in terms of the input-space residual $\mathbf{r}_t:=\mathbf{w}-\hat{\mathbf{w}}_t$, where $\hat{\mathbf{w}}_t:=\mathbf{u}^\top\boldsymbol{\theta}_t\in\mathbb{R}^M$. 
Then
\begin{equation}
\mathbf{r}_{t+1}
=\left(\mathbf{I}_M-\frac{\eta}{D}\,\mathbf{u}^\top\mathbf{u}\,\mathbf{X}\mathbf{X}^\top\right)\mathbf{r}_t,
\label{eq:GDequation}
\end{equation}
so convergence is controlled by the spectrum of $\mathbf{u}^\top\mathbf{u}\,\mathbf{X}\mathbf{X}^\top$\footnote{\JSCheck{The data dependence enters through the empirical covariance $\mathbf{X}\mathbf{X}^\top$. For dense inputs, $\operatorname{rank}(\mathbf{X}\mathbf{X}^\top)=\operatorname{rank}(\mathbf{X})\approx \min\{M,D\}$, which is typically $\approx D$ when $D\lesssim M$, whereas under our sparse model, $\operatorname{rank}(\mathbf{X}\mathbf{X}^\top)\le K(D)\asymp D^{1/(\alpha_1+1)}$, so it grows only sublinearly in $D$.}}
; see Appendix~\ref{appendix:GD}.}

\subsection{Convergence with High Probability}

\JSCheck{A first question is how $\eta$ must scale for \eqref{eq:GDequation} to converge.
Since $\ell(\boldsymbol{\theta})$ is quadratic, full-batch GD is a linear iteration, and a standard stability criterion implies that \eqref{eq:GDequation} converges whenever $0<\eta<\frac{2}{\lambda_{\max}\!\left(\frac{1}{D}\,\mathbf{u}^\top\mathbf{u}\,\mathbf{X}\mathbf{X}^\top\right)}$, where $\lambda_{\max}(\cdot)$ denotes the largest eigenvalue~\citep{nesterov2004introductory}. Equivalently, this is the usual condition $0<\eta<2/\lambda_{\max}(\nabla_{\boldsymbol{\theta}}^2\ell)$, since $\nabla_{\boldsymbol{\theta}}^2\ell(\boldsymbol{\theta})=\frac{1}{D}\mathbf{u}\mathbf{X}\mathbf{X}^\top\mathbf{u}^\top$.\footnote{Here $\mathbf{u}^\top\mathbf{u}$ and $\mathbf{X}\mathbf{X}^\top$ are positive semidefinite. Moreover, $\mathbf{u}^\top\mathbf{u}\,\mathbf{X}\mathbf{X}^\top$ has the same nonzero eigenvalues as the symmetric positive semidefinite matrix $(\mathbf{X}\mathbf{X}^\top)^{1/2}(\mathbf{u}^\top\mathbf{u})(\mathbf{X}\mathbf{X}^\top)^{1/2}$, so its spectrum on the data span is real and nonnegative.} 
In the underparameterized regime, a standard concentration argument for the feature-space Hessian $\nabla_{\boldsymbol{\theta}}^2 \ell(\boldsymbol{\theta})=\frac{1}{D}\,\boldsymbol{\Phi}\boldsymbol{\Phi}^\top$ yields $\lambda_{\max}(\nabla_{\boldsymbol{\theta}}^2 \ell)=1+o_{\mathbb{P}}(1)$ and hence any fixed $\eta\in(0,2)$ is admissible with high probability. The overparameterized case, however, is more subtle and is controlled by whether the random embedding acts isometrically on the data span.}

\begin{proposition}[Step-size Stability with High Probability]
\label{thm:gd_eta2_hpg}
Assume the sparse activation model of Section~\ref{Section:Ourmodel} with $\alpha_1\geq-1$ and
$\alpha_1+\alpha_2+1>0$. Assume further that the random feature map is sufficiently wide that it is
nearly isometric on the (low-dimensional) span of the data,\footnote{Concretely, since
$\mathrm{rank}(\mathbf{X}\mathbf{X}^\top)\le D$, standard random projection heuristics suggest that for
$N\gg D$ the operator $\mathbf{u}^\top\mathbf{u}$ acts approximately like $\mathbf{I}_M$ when sandwiched
against $\mathbf{X}\mathbf{X}^\top$.}
so that
\begin{equation}
\lambda_{\max}\!\left(\frac{1}{D}\,\mathbf{u}^\top\mathbf{u}\,\mathbf{X}\mathbf{X}^\top\right)
= 1 + o_{\mathbb{P}}(1).
\end{equation}
Consequently, for any fixed $\eta\in(0,2)$ the iteration \eqref{eq:GDequation} is stable and converges with
high probability.
\end{proposition}

\subsection{A Scaling Law for Failure of Gradient Descent}
\label{sec:gd_failure}

\JSCheck{Proposition~\ref{thm:gd_eta2_hpg} is a high-probability statement; rare datasets can still contain a single rare activation spike that produces an anomalously large top eigenvalue and causes divergence.\footnote{\JSCheck{Rare activations can create spectral outliers in $\mathbf{X}\mathbf{X}^\top$ that are not controlled by the high-probability spectral concentration arguments used in recent approximation/bias/variance analyses of linear-regression scaling laws~\citep{Lin2024LinRegScaling,Lin2025DataReuse,yan2025largerdatasetsrepeatedmore}, so we instead adopt a probabilistic treatment.}} GD becomes unstable when $\lambda_{\max}\!\bigl(\frac{\eta}{D}\mathbf{u}^\top\mathbf{u}\mathbf{X}\mathbf{X}^\top\bigr)>2$. Under the near-isometry approximation $\mathbf{u}^\top\mathbf{u}\approx \mathbf{I}_M$ on the data span, this reduces to $\lambda_{\max}\!\left(\mathbf{X}\mathbf{X}^\top\right) \;>\; \frac{2D}{\eta}$. This is only a concern when activated amplitudes \emph{grow} with index, \textit{i.e.}\ when $\alpha_2<-1$ (since then $x_j^2=j^{-(\alpha_2+1)}$ increases with $j$ on activation).}

\begin{wrapfigure}[25]{R}{0.525\textwidth}
\vspace{-0.7cm}
\centering
\includegraphics[width=0.525\textwidth]{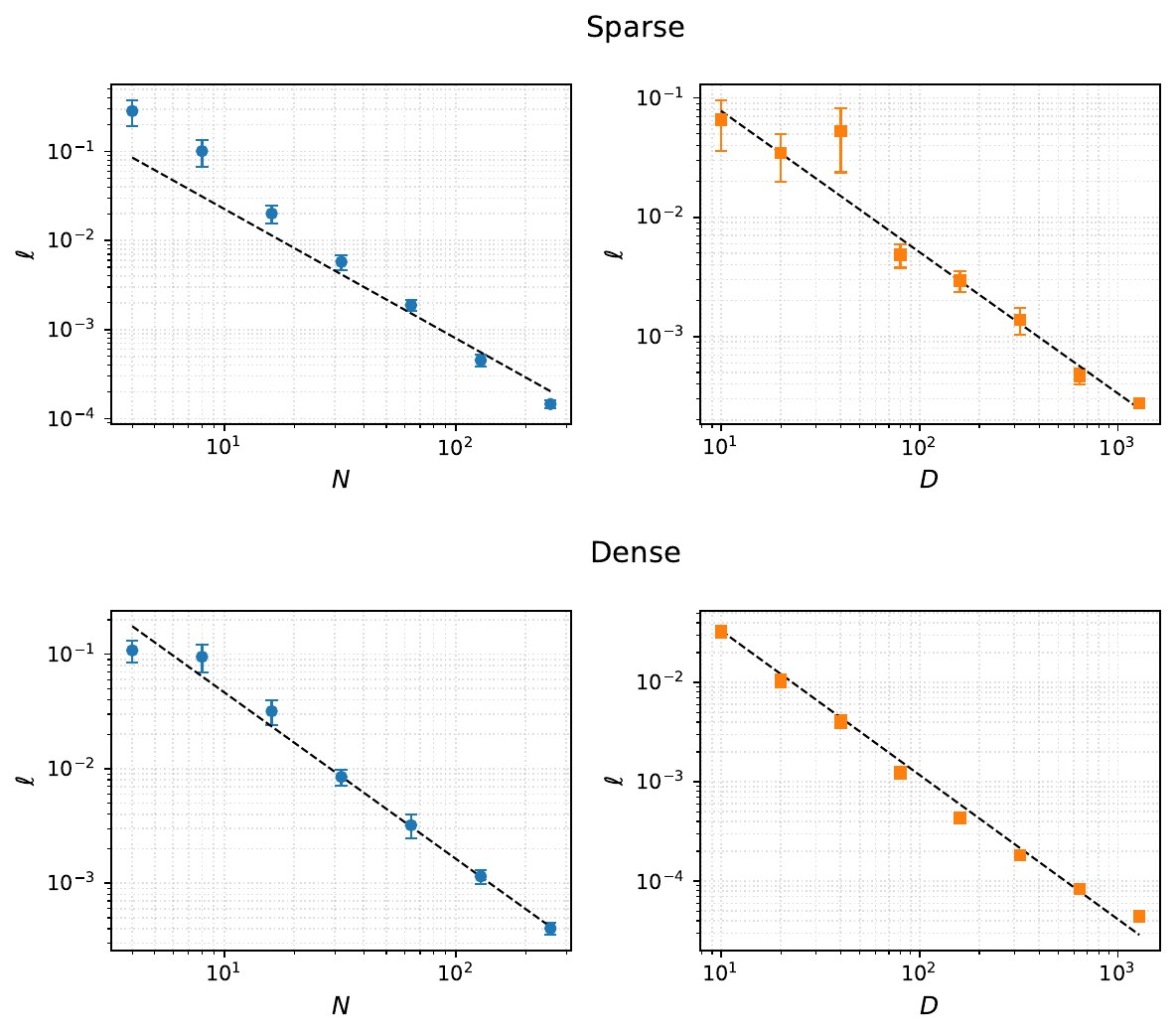}
\vspace{-6mm}
\caption{\textbf{Asymmetry persists under nonlinearity.} Test loss scaling under the ReLU feature map $\boldsymbol{\phi} = \sigma(\mathbf{u}\mathbf{x})$, trained with Nesterov + adaptive restart; $(\alpha_1, \alpha_2) = (1.0, 0.3)$, $5$ seeds. Top: sparse; bottom: dense. Left: $N$-sweep at $D = 50{,}000$; right: $D$-sweep at $N = 8000$. A single exponent $\alpha \approx 1.5$ (dashed lines, jointly fitted with separate intercepts) describes the sparse $N$-, dense $N$-, and dense $D$-sweeps. The sparse $D$-sweep requires its own shallower fit, $\alpha_D \approx 1.2$, breaking the dense-baseline symmetry $\alpha_N = \alpha_D$ predicted by linear theory.}
\label{fig:MLP}
\end{wrapfigure}

\begin{theorem}[Failure Probability of GD]
\label{thm:gd_failure_scaling}
Fix $\eta\in(0,2)$ and suppose $\alpha_2<-1$. Under the sparse activation model of
Section~\ref{Section:Ourmodel}, with $\alpha_1\geq-1$ and $\alpha_1+\alpha_2+1>0$, let
$\nu:=\frac{\alpha_1+\alpha_2+1}{-\alpha_2-1}>0$.
Then the probability that full-batch GD is unstable due to a rare spike on a random dataset of size $D$ obeys
\begin{equation}
\mathbb{P}_{\mathrm{rare}}(\mathrm{diverge})\asymp D^{-\nu},
\label{eq:gd_failure_scaling}
\end{equation}
up to $\eta$-dependent constants. 
\end{theorem}
\JSCheck{This theorem identifies instability under a fixed step size, not a fundamental optimization barrier. Reducing the learning rate or using clipped updates can restore stability~\citep{nguyen2023improved}.}

\section{Experiments with Nonlinear  Activations}

\JSCheck{A natural question is whether the two-exponent asymmetry we identify is specific to the linear setting or reflects a more general property of sparse data distributions. In Figure~\ref{fig:MLP}, we test robustness to nonlinearity by training a two-layer network with hidden layer $\boldsymbol{\phi} = \sigma(\mathbf{u}\mathbf{x})$ (ReLU activation, frozen random first-layer weights) and a linear readout, using full-batch accelerated GD with Nesterov momentum and adaptive restart; see Appendix~\ref{app:training} for experimental details. The fitted exponents shift downward in the model-limited regime, consistent with the general expectation that nonlinear feature maps smooth power-law input spectra~\citep{cho2009kernel,louart2018random,bietti2019inductive}. Crucially, the qualitative asymmetry persists: sparse $N$-, dense $N$-, dense $D$-sweeps are well-described by a single shared exponent, while the sparse $D$-sweep is better described by a distinct exponent. This supports our main result: the bottleneck mechanism---unobserved coordinates carrying no signal---is robust to nonlinearity and first-order optimization.}

\section{Conclusion}

\JSCheck{We introduced a model that gives rise to two distinct scaling exponents under sparse input activations. By analyzing the population loss in both under- and overparameterized regimes, we showed how sparsity breaks the symmetry of scaling laws in the random feature setting. We derived a compute-optimal scaling law showing that sparsity enhances efficiency by accelerating the decay of test loss with compute, with a frontier that prioritizes data over model size. Analyzing GD dynamics, we proved high-probability stability for all step sizes below 2 and identified a suppressed sublinear scaling law for its probability of failure in the overparameterized regime, suggesting a role for gradient clipping in practical optimization. We validated these results experimentally in the theoretical setting and extended them to a nonlinear feature map, confirming robustness of the sparsity-induced asymmetry to nonlinearity. Our main takeaway is that sparsity as a structural property fundamentally shapes scaling behavior in high-dimensional learning.}

\vspace{-3mm}
\paragraph{Limitations.} \JSCheck{Our results are built around data sparsity as the organizing principle, yet several extensions remain for future work. Incorporating exogenous label noise, which we expect to impose an error floor rather than remove the sparsity-driven bottleneck until it dominates the optimization scale, and analyzing ridge regularization, which would refine the statistical picture in the presence of explicit shrinkage, would extend our picture to noisy and explicitly regularized settings. More ambitiously, a theoretical treatment of nonlinear feature maps---complementing our empirical validation---and an extension to the dynamical regime of few- to multi-pass SGD both confront a common technical obstacle: the non-Gaussian, heavy-tailed nature of our features requires extending existing deterministic-equivalent tools beyond the Gaussian setting~\citep{mei2022rf,hu2022universality}.  Relatedly, extending our analysis to the feature-learning regime is an important direction, especially since feature learning may either concentrate capacity on rare informative coordinates~\citep{bordelon2025featurelearning} or induce sparse representations that hurt generalization in smooth tasks~\citep{petrini2022learning}. A further natural direction is to test these effects at scale, in deeper networks and in transformers, particularly in light of sparse-feature observations from sparse autoencoders~\citep{bricken2023monosemanticity,templeton2024scaling} and recent connections between superposition~\citep{elhage2022toy} and scaling laws~\citep{liu2025superposition}.}



\bibliographystyle{unsrtnat} 
\bibliography{references}

\newpage
\appendix

\newcommand{\weff}{\mathbf{w}_{\mathrm{eff}}}

\section{Additional Related Work}\label{appendix:AddRWork}
Empirical scaling laws relating test loss to parameters, data, and compute were systematically documented in language modeling \citep{Kaplan2020ScalingLaws} and refined into compute-optimal prescriptions emphasizing data-limited training in LLMs \citep{Hoffmann2022Chinchilla}. Related scaling behavior has also been observed across tasks and modalities, including \JSCheck{autoregressive generative modeling~\citep{henighan2020scaling}}, multi-domain studies \citep{Hestness2017PredictableScaling} and vision transformer scaling analyses \citep{Zhai2022ScalingViT}. More recent empirical work has expanded scaling-law analyses to address compute-optimal discrepancies across experimental protocols \citep{Porian2024ResolvingDiscrepancies}, practical scaling laws for weight decay and batch size in LLM pretraining \citep{bergsma2025powerlines}, precision-aware scaling under quantization and training precision \citep{Kumar2025ScalingPrecision}, inference-aware compute-optimal scaling \citep{Sardana2024BeyondChinchillaOptimal}, data-scarcity regimes via model reuse \citep{Wang2023DataEfficientScaling}, and data-constrained scaling under token repetition \citep{Muennighoff2025ScalingDataConstrained}.

\JSCheck{On the theory side, a growing literature develops tractable models that reproduce and explain scaling-law phenomenology. 
Statistical-physics teacher-student models characterize optimal errors and phase transitions in high-dimensional generalized linear models~\citep{barbier2019optimal} and computational-to-statistical gaps in committee machines~\citep{aubin2019committee}. 
Solvable neural-scaling models connect power-law learning curves to spectral structure and compute-optimal frontiers~\citep{Bahri2021ExplainingScaling,Maloney2022SolvableScaling}. 
Random-feature and kernel analyses address invariances~\citep{mei2021invariances}, hypercontractivity and kernel matrix concentration~\citep{mei2022generalization}, correlated linearizations and test error~\citep{LatourelleVigeant2023DysonRF}, and dimension-free deterministic equivalents~\citep{Defilippis2024DimensionFreeDE}. 
Scaling and renormalization approaches provide a complementary high-dimensional regression perspective~\citep{Atanasov2024ScalingRenorm}. 
Recent work also extends this power-law spectral perspective to kernel ridge regression under anisotropic power-law data~\citep{WortsmanLoureiro2025} and spectral inheritance through random-feature maps~\citep{PaquetteXiaoZhu2026}. 
Related statistical-physics analyses study feature learning and specialization in finite-width teacher-student neural networks near interpolation~\citep{barbier2025statistical}, as well as scaling-law phase diagrams in shallow feature-learning networks~\citep{defilippis2026scaling}. 
A complementary line studies compute-optimal scaling in linear regression~\citep{Lin2024LinRegScaling} and precision-aware variants in high-dimensional linear regression~\citep{zhang2026precision}. 
Recent work also studies the origin of scaling laws in simplified language and sequence-modeling settings, including theories based on natural-language statistics~\citep{cagnetta2026deriving} and controlled graph/random-walk generative processes~\citep{barkeshli2026origin}.}

\JSCheck{Beyond Bayes-optimal analyses, several works study \emph{training dynamics}. 
Dynamical models of one-pass or finite-time training exhibit multiple scaling phases and shift compute-optimal allocations~\citep{Bordelon2024DynamicalScaling,Paquette2024Phases}. 
Feature learning can also change scaling exponents and compute-optimal behavior~\citep{bordelon2025featurelearning}. 
Data reuse and multi-epoch training shift effective exponents and optimal allocations~\citep{Lin2025DataReuse,yan2025largerdatasetsrepeatedmore}. 
Functional loss dynamics and learning-rate schedules in kernel or random-feature regression are studied in~\citep{li2025functional,bordelon2026optimal}, while early stopping in random-feature regression with power-law spectra is analyzed in~\citep{kramp2026dynamics}. 
Related analyses of SGD in shallow or two-layer networks with heterogeneous teacher components~\citep{RenNichaniWuLee2025}, quadratic feature-learning models~\citep{benarous2025quadratic}, or power-law data spectra~\citep{WorschechRosenow2024} derive sample-, time-, and parameter-dependent scaling exponents.}

Our work is motivated by a different axis: \emph{sparse or conditional feature activation}~\citep{Glorot2011DeepSparseRectifier,Olshausen1996SparseCodingNature}, in which only a subset of features is observed or used for any example. Such sparsity is common in modern representation learning and conditional computation (\textit{e.g.}, mixture-of-experts)~\citep{Shazeer2017MoE}. \JSCheck{Related empirical work studies scaling laws for sparse conditional computation
in mixture-of-experts models~\citep{ludziejewski2024finegrainedmoe,abnar2025parameters}.} Rather than refining existing spectral-decay accounts, we isolate how sparsity in what the dataset reveals can itself produce an intrinsic asymmetry between model-limited and data-limited scaling exponents.

\section{\JSCheck{Population Objective of Sparse Model}}\label{appendix:computation}

We begin by reformulating the objective in Eq.~\eqref{eq:bayes-loss} as the following loss function:
\begin{equation}
\begin{split}
    \ell \;&=\; \frac{\gamma}{D} \|\boldsymbol{\theta}\|_2^2 
    + \frac{1}{D} \sum_{a=1}^D \Bigl(y(\mathbf{x}_a) - \boldsymbol{\theta}^\top \mathbf{u} \mathbf{x}_a\Bigr)^2 \\
    &=\; \frac{\gamma}{D} \|\boldsymbol{\theta}\|_2^2 
    + \frac{1}{D} \sum_{a=1}^D \Bigl(\mathbf{w}^\top \mathbf{x}_a - \boldsymbol{\theta}^\top \mathbf{u} \mathbf{x}_a\Bigr)^2 \\
    &=\; \frac{\gamma}{D} \|\boldsymbol{\theta}\|_2^2 
    + \frac{1}{D} \sum_{a=1}^D \Bigl[
        \mathbf{w}^\top \mathbf{x}_a \mathbf{x}_a^\top \mathbf{w}
        - 2\, \boldsymbol{\theta}^\top \mathbf{u} \mathbf{x}_a \mathbf{x}_a^\top \mathbf{w}
        + \boldsymbol{\theta}^\top \mathbf{u} \mathbf{x}_a \mathbf{x}_a^\top \mathbf{u}^\top \boldsymbol{\theta}
    \Bigr], 
\end{split}
\label{eq:lossEquation}
\end{equation}
which we wish to minimize with respect to $\boldsymbol{\theta}$. While our focus is on the ridge-less limit, $\gamma=0$, we retain it for mathematical convenience and take the limit $\gamma \to 0$ when appropriate.

Using $\boldsymbol{\Phi} = \mathbf{u} \mathbf{X}$, we find the optimal solution $\boldsymbol{\theta}^\star$ satisfies:
\begin{equation}
    \boldsymbol{\theta}^{\star\!\top} = \mathbf{y} \boldsymbol{\Phi}^\top \mathbf{q} = \mathbf{y} \mathbf{Q} \boldsymbol{\Phi}^\top,
    \label{eq:thetaSolved}
\end{equation}
where
\begin{equation*}
\begin{split}
    \mathbf{q} &= (\gamma\, \mathbf{I}_N + \boldsymbol{\Phi} \boldsymbol{\Phi}^\top)^{-1}, \\
    \mathbf{Q} &= (\gamma\, \mathbf{I}_D + \boldsymbol{\Phi}^\top \boldsymbol{\Phi})^{-1}.
\end{split}
\end{equation*}
We are interested in the difference between the effective weight vector $\mathbf{w}_{\text{eff}}$, defined by
\[
\mathbf{w}_{\text{eff}}^{\top} = \boldsymbol{\theta}^{\star\!\top} \mathbf{u},
\]
and the true target weights $\mathbf{w}^{\top}$. In particular, the generalization error is given by
\begin{equation*}
    \mathcal{E}_{\text{gen}} = 
    \bigl(\mathbf{w}_{\text{eff}}^{\top} - \mathbf{w}^{\top}\bigr)
    \,\boldsymbol{\Sigma}_{\mathbf{x}}\,
    \bigl(\mathbf{w}_{\text{eff}} - \mathbf{w}\bigr).
\end{equation*}

Table~\ref{tab:matrixTable2} lists the shapes and elementwise scalings used in our computations.
\begin{table}[h]
\vspace{-3mm}
\caption{Shapes and elementwise scaling.}
\label{tab:matrixTable2}
\small
\setlength{\tabcolsep}{4pt}
\renewcommand{\arraystretch}{1.15}
\begin{tabularx}{\columnwidth}{
  >{\hsize=0.3\hsize\centering\arraybackslash}X
  >{\hsize=1.7\hsize\centering\arraybackslash}X
  >{\hsize=0.3\hsize\centering\arraybackslash}X
  >{\hsize=1.7\hsize\centering\arraybackslash}X
}
\toprule
Parameter & Description & Shape & Scaling (elementwise) \\
\midrule
$\mathbf{X}$ & Training data (columns are examples) & $M\times D$ & per-coordinate standard deviation $j^{-(\alpha_1+\alpha_2+2)/2}$ \\
$\mathbf{u}$ & Fixed random embedding & $N\times M$ & $u_{ij}\sim\mathcal{N}(0,1/N)$ i.i.d. \\
$\boldsymbol{\theta}$ & Trainable readout & $N\times 1$ & $O(1)$ \\
$\mathbf{w}$ & Target linear functional & $M\times 1$ & $O(1)$ \\
$\mathbf{y}$ & Target function & $1\times D$ & $O(1)$ \\
$\boldsymbol{\Phi}$ & Feature matrix $\mathbf{u}\mathbf{X}$ & $N\times D$ & N/A \\
$\mathbf{Q}$       & Feature Gram $(\gamma\mathbf{I}_D + \boldsymbol{\Phi}^{\top}\boldsymbol{\Phi})^{-1}$ & $D\times D$ & N/A\\
$\mathbf{q}$       & Dual Gram $(\gamma\mathbf{I}_N + \boldsymbol{\Phi}\boldsymbol{\Phi}^{\top})^{-1}$ & $N\times N$ & N/A\\
\bottomrule
\end{tabularx}
\vspace{-2mm}
\end{table}

\section{Scaling Asymptotics of Sparse Random Features}\label{App:Scaling}

\subsection{$-1\le \alpha_1 \le 0$: One-Exponent Scaling Law}\label{app:weak-sparsity}

In our sparse model of Section~\ref{Section:Ourmodel}, we obtain a symmetric one-exponent scaling law when $-1\le \alpha_1 \le 0$, similar to dense random feature models~\citep{Bahri2021ExplainingScaling, Maloney2022SolvableScaling}. The key point is that generalization can be bottlenecked by three resources: the number of parameters $N$, the number of samples $D$, and the set of input coordinates that are actually observed (activated at least once) in the $D$ training examples.

Recall that the activation probability of coordinate $j$ is
\[
p_j := \mathbb{P}(x_j\neq 0)=j^{-(\alpha_1+1)}.
\]
Across $D$ i.i.d.\ samples, let $N_j$ denote the number of times coordinate $j$ activates. Then
\[
\mathbb{P}(\text{coordinate $j$ is never active in the training set})
=\mathbb{P}(N_j=0)=(1-p_j)^D \approx e^{-Dp_j}.
\]
Thus, whenever $Dp_j\gg 1$ (equivalently, $D\,j^{-(\alpha_1+1)}\gg 1$), we have $\mathbb{P}(N_j=0)\le e^{-\Omega(1)}\ll 1$, so coordinate $j$ is activated at least once with high probability. \JSCheck{When $-1\le \alpha_1<0$, this condition holds uniformly for all $j$ up to order $D$. At the boundary $\alpha_1=0$, the coordinate $j\sim D$ has $Dp_j=O(1)$, so the never-observed probability is also $O(1)$; this boundary case only affects constants and does not change the one-exponent scaling.}


Therefore the Bayes-optimal loss is governed by the unmodeled-variance tail beyond the first $\min\{N,D\}$ coordinates:
\[
\ell_{\mathrm{Bayes}}(N,D)\;\asymp\;\sum_{j>\min\{N,D\}} \operatorname{Var}(x_j)
\;\asymp\;\sum_{j>\min\{N,D\}} j^{-(\alpha_1+\alpha_2+2)}
\;\asymp\;\min\{N,D\}^{-(\alpha_1+\alpha_2+1)},
\]
and the scaling is symmetric in $N$ and $D$ in this regime, with a single exponent $\alpha_N = \alpha_D = \alpha_1+\alpha_2+1$.

\subsection{$\alpha_1>0$: Two-Exponent Scaling Law}\label{appendix:TwoExpon}
Here we briefly elaborate on supplementary details omitted from the main text regarding the scaling asymptotics of the two-exponent scaling law.

In a random set of $D$ data points, only $O(D^{\frac{1}{1+\alpha_1}})$ features are likely to activate when $\alpha_1 > 0$. Thus, for positive $\alpha_1$, there are naively three distinct regimes:
\[
N < D^{\frac{1}{1+\alpha_1}}, \quad D^{\frac{1}{1+\alpha_1}} < N < D, \quad \text{and} \quad N > D.
\]
The first is unambiguously underparameterized, the third is clearly overparameterized, and the intermediate regime exhibits characteristics of both. 

\subsubsection{Underparameterized Regime}
We first consider the underparameterized regime \( N \ll D^{\frac{1}{1+\alpha_1}} \),  where learning is primarily limited by model capacity.

\paragraph{Proof of Proposition~\ref{prop:underparam}}

\begin{proof} A resolvent analysis of random-feature regression~\citep{Maloney2022SolvableScaling} reveals that the loss is governed by the covariance mass missed by an $N$-dimensional random feature subspace. In the ridgeless limit, the resolvent acts as a projection onto the latent directions inaccessible to the random features, so the leading-order loss is controlled by the unresolved tail of the covariance spectrum. Thus, at the level of scaling, the random-feature bottleneck resolves the leading $N$ spectral directions and leaves the remaining covariance tail unresolved. In the present regime $N\ll D^{1/(1+\alpha_1)}$, sparsity is masked, so the relevant spectrum is $\operatorname{Var}(x_j)=j^{-\alpha_1-\alpha_2-2}$. Hence the residual variance scales as
\begin{align*}
\sum_{j > N} j^{-\alpha_1 - \alpha_2 - 2}
&\;\asymp\;
\int_N^\infty j^{-\alpha_1 - \alpha_2 - 2}\, dj \\
&=
\frac{1}{\alpha_1 + \alpha_2 + 1}
\, N^{-(\alpha_1 + \alpha_2 + 1)}.
\end{align*}
Combining this with the general form of the Bayes-optimal loss 
(Eq.~\eqref{eq:bayes-tail}) yields~\eqref{eq:bayes-N-scaling-prop}.
\end{proof}

\subsubsection{Overparameterized Regime}
We now turn to the overparameterized regime \( N \gg D^{\frac{1}{1+\alpha_1}} \).

\paragraph{Heavily Overparameterized Regime}
We first focus on the maximally overparameterized case. We will show that in this regime, the performance of the random feature model is Bayes-optimal among all models, and we compute its generalization error.
 
\subparagraph{\emph{\JSCheck{Optimal Predictor}.}}

An upper bound on the performance of any model would be for $\weff$ to equal $\mathbf{w}$ on indices that appear in the dataset, and be $0$ on all other indices. We will show that this is indeed what happens in this regime.

If we write out our expression for $\weff$, we have
\begin{equation}
    \weff^{\top}=\mathbf{w}^{\top} \,\mathbf{X}\,\bigl(\gamma \mathbf{I}_D+\mathbf{X}^{\top} \mathbf{u}^{\top}\mathbf{u}\,\mathbf{X}\bigr)^{-1}\,\mathbf{X}^{\top} \mathbf{u}^{\top} \mathbf{u},
    \label{eq:weffEquation4Real}
\end{equation}
If we post-multiply by $\mathbf{X}\,\mathbf{b}$ for any vector $\mathbf{b}$, we obtain
\begin{equation*}
    \weff^{\top} \,\mathbf{X}\,\mathbf{b}=\mathbf{w}^{\top}\,\mathbf{X}\,\mathbf{b}.
\end{equation*}
This follows from the full rank of $\mathbf{Q}$. Thus, in the overparameterized regime, for every possible input in the subspace spanned by the training data, the random feature model incurs zero loss.

\subparagraph{\emph{Coefficients Outside the Training Data Go to $0$.}}
\JSCheck{While the analysis showing that \(\weff\) perfectly interpolates the training data applies for any overparameterized model in the limit \(\gamma\to0\), the argument in this subsection uses the stronger assumption that we are deep in the overparameterized regime. Since \(u_{ij}\sim\mathcal{N}(0,1/N)\), the diagonal entries of \(\mathbf{u}^\top\mathbf{u}\) are \(1+O_{\mathbb{P}}(N^{-1/2})\), while the off-diagonal entries are \(O_{\mathbb{P}}(N^{-1/2})\). Thus, for any subspace \(S\) with \(\dim(S)\ll N\), the restriction of \(\mathbf{u}^\top\mathbf{u}\) to \(S\) is \(\mathbf{I}_S+o_{\mathbb{P}}(1)\). Since the row and column spaces of \(\mathbf{X}\) are at most \(D\)-dimensional, when \(D\ll N\) we can replace the occurrences of \(\mathbf{u}^\top\mathbf{u}\) in Eq.~\eqref{eq:weffEquation4Real} by \(\mathbf{I}\). Therefore, \(\weff\) simplifies to}
\begin{equation*}
    \weff^{\top}=\mathbf{w}^{\top} \,\mathbf{X}\,(\gamma \mathbf{I}_D+\mathbf{X}^{\top}\mathbf{X})^{-1}\,\mathbf{X}^{\top}.
    \label{eq:weffEquation-simplified}
\end{equation*}
In the limit $\gamma\to 0$, this is precisely the projection of $\mathbf{w}^{\top}$ onto the space spanned by the data. We can see this using a singular value decomposition.

\noindent\emph{Proof.}
Let the singular value decomposition (SVD) of the data matrix be $\mathbf{X} = U \Sigma V^\top$, where $\Sigma$ is the diagonal matrix of singular values $\sigma_i$. Substituting this into the term acting on $\mathbf{w}^\top$ yields
\begin{align*}
\mathbf{X}\,(\gamma \mathbf{I}_D+\mathbf{X}^\top \mathbf{X})^{-1}\mathbf{X}^\top
&= U \Sigma V^\top \bigl(\gamma \mathbf{I}_D + V \Sigma^\top \Sigma V^\top \bigr)^{-1} V \Sigma^\top U^\top \\
&= U \Sigma V^\top \bigl[ V (\gamma \mathbf{I}_D + \Sigma^\top \Sigma) V^\top \bigr]^{-1} V \Sigma^\top U^\top \\
&= U \underbrace{\Sigma (\gamma \mathbf{I}_D + \Sigma^\top \Sigma)^{-1} \Sigma^\top}_{\Lambda_\gamma} U^\top.
\end{align*}
The inner matrix $\Lambda_\gamma$ is diagonal with entries $(\Lambda_\gamma)_{ii} = \sigma_i^2/(\gamma + \sigma_i^2)$. Taking the limit $\gamma \to 0$ we find
\[
\lim_{\gamma \to 0} \frac{\sigma_i^2}{\gamma + \sigma_i^2} =
\begin{cases}
1, & \text{if } \sigma_i \neq 0,\\
0, & \text{if } \sigma_i = 0.
\end{cases}
\]
Thus, $\lim_{\gamma \to 0} \Lambda_\gamma$ acts as the identity on the subspace associated with non-zero singular values. Consequently, the full expression becomes $U_{\mathrm{range}} U_{\mathrm{range}}^\top$, which is the orthogonal projection operator onto the column space of $\mathbf{X}$. Therefore, in this limit, $\weff^{\top}$ is precisely the projection of $\mathbf{w}^{\top}$ onto the space spanned by the training data.

\subparagraph{\emph{Expected Test \JSCheck{Loss}.}}

In the heavily overparameterized limit, the distribution of $\mathbf{w}-\weff$ is straightforward. The $j$-th coefficient is zero if feature $j$ appears at least once in the dataset, and is normally distributed with variance~$1$ if feature $j$ does not appear. The contribution to the test \JSCheck{loss} is thus
\begin{equation*}
    \ell_{\mathrm{test}}=\mathcal{E}_{\text{gen}}=\sum_j \Bigl(\mathbb{E}_{\mathrm{test}}[x_j^2]\Bigr)\,\Bigl(\mathbb{E}_{\mathrm{train}}\bigl[(\mathbf{w}-\weff)_j^2\bigr]\Bigr).
\end{equation*}

\paragraph{Proof of Lemma~\ref{lem:learnable-coordinates}}

Here we derive the asymptotic expression for the expected number of coordinates that
are active at least once across $D$ i.i.d.\ samples under the sparse activation
model.

\begin{proof}[Proof (learnable-coordinate count $K(D)$)]
In the sparse activation model, coordinate $j$ is active in a single sample
with probability $p_j = j^{-\alpha_1 - 1}$ and inactive with probability
$1 - p_j$. The probability that coordinate $j$ is never activated in $D$
independent samples is
\[
\mathbb{P}_{\mathrm{never}}(j)
= (1 - p_j)^D.
\]
For $p_j \ll 1$, we use $(1 - p_j)^D \approx e^{-D p_j}$ as $D \to \infty$; this
follows from $(1 - \frac{D p_j}{D})^D \to e^{-D p_j}$. Hence
\[
\mathbb{P}_{\mathrm{never}}(j)
\approx \exp(-D p_j)
= \exp\!\left(-D j^{-\alpha_1 - 1}\right).
\]
Thus the probability that coordinate $j$ is activated at least once is
$1 - \exp(-D j^{-\alpha_1 - 1})$, and the expected number of such coordinates
is
\[
K(D)
\;\approx\;
\sum_{j=1}^\infty
\bigl(1 - e^{-D j^{-\alpha_1-1}}\bigr).
\]
Passing to the continuum limit, we obtain
\[
K(D)
\;\asymp\;
\int_0^\infty \bigl(1 - e^{-D j^{-\alpha_1-1}}\bigr)\, dj.
\]
We now perform the change of variables
\[
t = D j^{-(\alpha_1+1)}
\quad\Longrightarrow\quad
j = \left(\frac{D}{t}\right)^{\frac{1}{\alpha_1+1}},
\quad
dj
= -\frac{1}{\alpha_1+1}
D^{\frac{1}{\alpha_1+1}}
t^{-\frac{\alpha_1+2}{\alpha_1+1}}\, dt,
\]
which yields
\begin{align*}
K(D)
&\;\asymp\;
\frac{1}{\alpha_1+1}
D^{\frac{1}{\alpha_1+1}}
\int_0^\infty \bigl(1 - e^{-t}\bigr) t^{-\frac{\alpha_1+2}{\alpha_1+1}} \, dt.
\end{align*}
Let $\beta \equiv \frac{1}{\alpha_1+1}\in(0,1)$. By integration by parts,
\[
\int_0^\infty (1-e^{-t})\, t^{-1-\beta}\,dt
= \frac{1}{\beta}\int_0^\infty e^{-t} t^{-\beta}\,dt
= \frac{\Gamma(1-\beta)}{\beta}.
\]
Substituting $\beta=\frac{1}{\alpha_1+1}$ gives
\[
K(D)
\;\asymp\;
\Gamma\!\left(1-\frac{1}{\alpha_1+1}\right)\,
D^{\frac{1}{\alpha_1+1}}.
\]
This matches Eq.~\eqref{eq:K-of-D} in Lemma~\ref{lem:learnable-coordinates}.
\end{proof}

\paragraph{Proof of Theorem~\ref{thm:over-scaling}}\label{App:rig-proof}
\JSCheck{Here we give two derivations of Theorem~\ref{thm:over-scaling}: a discrete argument via the cutoff $K(D)$, which provides useful intuition, and a fully rigorous continuous argument.}

\begin{proof}[\JSCheck{Informal discrete argument via $K(D)$}]
By Lemma~\ref{lem:learnable-coordinates}, when $\alpha_1>0$, at most $K(D)$ coordinates of the
target vector $\mathbf{w}$ can be reliably estimated from $D$ samples. All
coordinates with indices $j > K(D)$ remain unobserved and therefore unmodeled.
By the unmodeled-variance principle (Eq.~\eqref{eq:bayes-tail}), the
Bayes-optimal loss is
\[
\ell_{\mathrm{Bayes},D}
\;\asymp\;
\sum_{j > K(D)} j^{-\alpha_1 - \alpha_2 - 2}.
\]
Approximating the sum by an integral,
\[
\begin{aligned}
\sum_{j > K(D)} j^{-\alpha_1 - \alpha_2 - 2}
&\;\asymp\;
\int_{K(D)}^\infty j^{-\alpha_1 - \alpha_2 - 2}\, dj
\\
&=
\frac{1}{\alpha_1 + \alpha_2 + 1}\,
K(D)^{-(\alpha_1+\alpha_2+1)}.
\end{aligned}
\]
Substituting the scaling $K(D) \asymp D^{1/(\alpha_1+1)}$ from
Eq.~\eqref{eq:K-of-D} yields
\[
\ell_{\mathrm{Bayes},D}
\;\asymp\;
D^{-(\alpha_1+\alpha_2+1)/(\alpha_1+1)},
\]
which is Eq.~\eqref{eq:bayes-loss-over}.
\end{proof}

\begin{proof}[\JSCheck{Rigorous} proof (continuous argument)]
An equivalent expression for the Bayes-optimal loss can be written as a continuous expectation over coordinates, weighted by the probability that each
coordinate is never activated. This leads to the integral representation
\begin{equation}
\ell_{\mathrm{Bayes},D}
=
\int_{0}^{\infty}
\exp\!\left(-D j^{-\alpha_1 - 1}\right)\,
j^{-(\alpha_1 + \alpha_2 + 2)}\, dj,
\label{eq:bayes-integral-over}
\end{equation}

We again apply the change of variables
\[
t = D j^{-(\alpha_1+1)}
\quad\Longrightarrow\quad
j = \left(\frac{D}{t}\right)^{\frac{1}{\alpha_1+1}},
\quad
dj
= -\frac{1}{\alpha_1+1}
D^{\frac{1}{\alpha_1+1}}
t^{-\frac{\alpha_1+2}{\alpha_1+1}}\, dt.
\]
Substituting into \eqref{eq:bayes-integral-over} gives
\[
\ell_{\mathrm{Bayes},D}
=
\frac{1}{\alpha_1+1}
D^{-\frac{\alpha_1+\alpha_2+1}{\alpha_1+1}}
\int_0^\infty
e^{-t}\,
t^{\frac{\alpha_1+\alpha_2+1}{\alpha_1+1}-1}\, dt.
\]
The remaining integral is exactly the Gamma function,
\[
\int_0^\infty e^{-t} t^{\beta-1}\, dt = \Gamma(\beta),
\]
with
\[
\beta = \frac{\alpha_1+\alpha_2+1}{\alpha_1+1}.
\]
Hence
\[
\ell_{\mathrm{Bayes},D}
=
\frac{1}{\alpha_1+1}
\Gamma\!\left(
\frac{\alpha_1+\alpha_2+1}{\alpha_1+1}
\right)
D^{-\frac{\alpha_1+\alpha_2+1}{\alpha_1+1}},
\]
which matches the scaling law in Eq.~\eqref{eq:bayes-loss-over} in the main text. When $\alpha_1>0$, sparsity reduces the effective number of observed coordinates to scale as $D^{1/(\alpha_1+1)}$, thereby becoming the dominant data bottleneck. This concludes the proof.
\end{proof}

\noindent\emph{Subtleties Concerning Theorem~\ref{thm:over-scaling}: Co-activation and Identifiability.} A subtle complication arises when two rare coordinates co-activate, \textit{i.e.}, are both nonzero in the same training example. In such cases, it becomes impossible to disentangle their individual contributions to the label, creating ambiguity in estimating the corresponding weights. The coordinates that activate exactly once have indices on the order of
$j \sim D^{\frac{1}{\alpha_1 + 1}}$. Consequently, the total number of such
coordinates also scales as $D^{\frac{1}{\alpha_1 + 1}}$. A nonzero number of
pairwise co-activations is therefore likely whenever this quantity exceeds
$D^{1/2}$ (by the birthday paradox), which occurs when $\alpha_1 < 1$.
However, the number of co-activating coordinates itself scales only as
\[
\min\left\{ \frac{\left(D^{\frac{1}{\alpha_1 + 1}}\right)^2}{D},\;
D^{\frac{1}{\alpha_1 + 1}} \right\}
= \min\left\{ D^{\frac{2}{\alpha_1 + 1} - 1},\;
D^{\frac{1}{\alpha_1 + 1}} \right\}
= \min\left\{ D^{\frac{1 - \alpha_1}{\alpha_1 + 1}},\;
D^{\frac{1}{\alpha_1 + 1}} \right\}.
\]
Each such coordinate contributes variance on the order of
$D^{\frac{-\alpha_1 - \alpha_2 - 2}{\alpha_1 + 1}}$, so the total variance
contributed by co-activating coordinates scales as
\[
\min\left\{
D^{\frac{-2\alpha_1 - \alpha_2 - 1}{\alpha_1 + 1}},\;
D^{\frac{-\alpha_1 - \alpha_2 - 1}{\alpha_1 + 1}}
\right\},
\]
which vanishes as $D \to \infty$. Thus, while co-activation may create isolated ambiguity, it does not affect the asymptotic loss scaling in Theorem~\ref{thm:over-scaling}.

\section{\JSCheck{Compute}}\label{appendix:compute}

This section provides an overview of the compute model we adopt for studying the compute-optimal frontier under training to completion:
\begin{equation}
    C = ND \cdot \min\{N,D\}.
\end{equation}

\subsection{Computational Efficiency}

Here, ``compute'' refers to the leading-order training cost, measured in FLOPs up to hardware-dependent constants. To approach the population-loss scaling studied in the main text, we focus on the setting in which the readout is trained to convergence. This can be done either with gradient-based methods, which are closer to practical training, or by directly computing the min-norm solution in Eq.~\eqref{eq:thetaSolved}.

\subsubsection{Gradient-Based Methods}

A single gradient step on the loss~\eqref{eq:lossEquation} costs $O(ND)$. The iteration complexity is controlled by the condition number of the feature Gram matrix, defined on its nonzero spectrum,
$\kappa(\boldsymbol{\Phi}\boldsymbol{\Phi}^\top)
=
\lambda_{\max}(\boldsymbol{\Phi}\boldsymbol{\Phi}^\top)/
\lambda_{\min}^+(\boldsymbol{\Phi}\boldsymbol{\Phi}^\top)$,
where $\lambda_{\min}^+$ denotes the smallest nonzero eigenvalue.
For plain GD, the number of steps scales linearly in $\kappa$, up to logarithmic accuracy factors. Accelerated first-order methods improve this dependence to $\sqrt{\kappa}$. We therefore write the leading gradient-based training cost as
\begin{equation}
    C_{\mathrm{grad}} \sim ND\,\kappa^p,
\end{equation}
with $p=1$ for GDand $p=1/2$ for accelerated methods.

\paragraph{Condition number, $\kappa$.}
We estimate the condition number from the effective variance profile of the active features. The number of active coordinates in a dataset of size $D$ is $K(D)$, so the readout effectively uses roughly $\min\{N,K(D)\}$ features. The variance explained by the smallest feature scales as $\min\{N,K(D)\}^{-(\alpha_1+\alpha_2+2)}$, while the largest is $O(1)$.
Thus the effective condition number scales as
\begin{equation}
\kappa \sim \min\{N,K(D)\}^{\alpha_1+\alpha_2+2}.
\end{equation}
Consequently, if the iteration complexity scales as \(\kappa^p\), the number of optimization steps scales as $\min\{N,K(D)\}^{(\alpha_1+\alpha_2+2)p}$.

\subsubsection{Direct Least-Squares Solve}
Eq.~\eqref{eq:thetaSolved} gives two equivalent expressions for $\boldsymbol{\theta}^\star$, related by a standard matrix identity. Using the first expression, $\mathbf{y}\,\boldsymbol{\Phi}^\top \mathbf{q}$, requires computing $\boldsymbol{\Phi}\boldsymbol{\Phi}^\top$ (cost $N^2D$), inverting it (cost $N^3$) to form $\mathbf{q}$, computing $\mathbf{y}\boldsymbol{\Phi}^\top$ (cost $ND$), and then multiplying the resulting row vector by $\mathbf{q}$ (cost $N^2$). The total cost is therefore $O(N^2D+N^3)$. By the same logic, using the second expression, $\mathbf{y}\,\mathbf{Q}\,\boldsymbol{\Phi}^\top$, costs $O(D^2N+D^3)$. The cheaper of these two options is
\[
C_{\mathrm{direct}}  \sim ND \cdot \min\{N,D\}.
\]

We summarize these estimates in Table~\ref{tab:threeRegimes}, which compares the costs of gradient descent (GD), accelerated first-order methods, and direct solvers across the three regimes.
\begin{table}[H]
\vspace{-4mm}
    \centering
    \caption{Compute cost, normalized by $ND$, across the three parameter regimes. The GD and accelerated columns assume iteration complexity scaling as $\kappa^p$, with $p=1$ for GD and $p=1/2$ for accelerated first-order methods such as Nesterov acceleration. The direct-solve column corresponds to computing the min-norm least-squares solution in closed form.}
    \label{tab:threeRegimes}
    \small
    \setlength{\tabcolsep}{4pt}
    \renewcommand{\arraystretch}{1.15}
    \begin{tabularx}{\columnwidth}{
        >{\hsize=0.8\hsize\centering\arraybackslash}X
        >{\hsize=0.4\hsize\centering\arraybackslash}X
        >{\hsize=0.4\hsize\centering\arraybackslash}X
        >{\hsize=0.4\hsize\centering\arraybackslash}X
        }
        \toprule
         Regime 
         & GD 
         & Accelerated 
         & Direct solve \\
         \midrule
         Heavily overparameterized, $N>D$ 
         & $D^{\frac{\alpha_1+\alpha_2+2}{\alpha_1+1}}$ 
         & $D^{\frac{\alpha_1+\alpha_2+2}{2(\alpha_1+1)}}$ 
         & $D$ \\
         Intermediate, $D>N>K(D)$ 
         & $D^{\frac{\alpha_1+\alpha_2+2}{\alpha_1+1}}$ 
         & $D^{\frac{\alpha_1+\alpha_2+2}{2(\alpha_1+1)}}$ 
         & $N$ \\
         Underparameterized, $K(D)>N$ 
         & $N^{\alpha_1+\alpha_2+2}$ 
         & $N^{\frac{\alpha_1+\alpha_2+2}{2}}$ 
         & $N$ \\
         \bottomrule
    \end{tabularx}
\end{table}

\subsection{Proxy Compute Model}

Table~\ref{tab:threeRegimes} implies different computational scalings depending on both the solver and the regime.

Let $a=\alpha_1+\alpha_2+2$ and $b=\alpha_1+1$, with $b>0$. The comparison between direct solvers and first-order methods is solver-dependent. Since accelerated methods reduce the condition-number dependence from $\kappa$ to $\kappa^{1/2}$, they are asymptotically cheaper than GD whenever $a>0$, up to constants. The direct solve is cheaper than GD in the heavily overparameterized and intermediate regimes when $a/b>1$, equivalently $\alpha_2+1>0$, while in the underparameterized regime the finite-variance condition $a>1$ already guarantees that the direct solve is cheaper than GD.

The comparison with accelerated methods is stricter. In the heavily overparameterized regime, the direct solve beats acceleration only when $a/(2b)>1$, equivalently $\alpha_2>\alpha_1$. In the underparameterized regime, it beats acceleration only when $a/2>1$, equivalently $\alpha_1+\alpha_2>0$. Thus direct solvers are uniformly cheapest under the stronger condition
\begin{equation}
\alpha_2>\alpha_1
\qquad\text{and}\qquad
\alpha_1+\alpha_2>0.
\end{equation}
There is also a mixed regime in which GD is the most expensive method, but the direct solve beats accelerated methods only in part of the phase diagram. For example, when $\alpha_1>0$ and
\begin{equation}
\max\{-1,-\alpha_1\}<\alpha_2<\alpha_1,
\end{equation}
acceleration is cheaper than the direct solve in the heavily overparameterized regime, while the direct solve is cheaper in the underparameterized regime. In the intermediate regime, the crossover occurs at $N_c(D)\sim D^{(\alpha_1+\alpha_2+2)/(2(\alpha_1+1))}$: direct solution is cheaper for $N<N_c(D)$, while acceleration is cheaper for $N>N_c(D)$.

For example, the representative choice $(\alpha_1,\alpha_2)=(1,0.3)$ used in our experiments lies in this mixed regime: GD is asymptotically the most expensive method, while the cheapest method switches from acceleration in the heavily overparameterized regime to direct solution in the underparameterized regime, with a crossover between the two in the intermediate regime.

This solver dependence motivates using a simple unified compute proxy,
\begin{equation}
C=ND\cdot \min\{N,D\},
\label{eq:proxy}
\end{equation}
rather than tying the main scaling analysis to a particular optimization algorithm. The qualitative conclusion we emphasize below is that sparsity shifts the compute-optimal allocation toward data, and this conclusion does not depend on these solver-dependent distinctions within reasonable compute models.

\subsubsection{Robustness to Alternative Compute Models}

The solver-dependent comparisons above change the mapping from $(N,D)$ to compute, but not the basic allocation principle. The loss has the form
\begin{equation}
\ell(N,D)\sim N^{-\alpha_N}+D^{-\alpha_D},
\end{equation}
so the compute-optimal allocation balances the two terms:
\begin{equation}
N^{-\alpha_N}\sim D^{-\alpha_D}
\qquad\Longrightarrow\qquad
D\sim N^{\alpha_N/\alpha_D}=N^{\alpha_1+1}.
\end{equation}
Thus sparsity fixes the relative allocation between data and model size independently of the particular solver. Different compute models only change how this balanced allocation scales with the total budget $C$.

For example, suppose that in the underparameterized branch the compute model takes the more general form
\begin{equation}
C \sim N^rD .
\end{equation}
Combining this with $D\sim N^{\alpha_1+1}$ gives
\begin{equation}
N^*(C)\sim C^{1/(r+\alpha_1+1)},
\qquad
D^*(C)\sim C^{(\alpha_1+1)/(r+\alpha_1+1)},
\end{equation}
and
\begin{equation}
\ell^*(C)\sim C^{-\alpha_N/(r+\alpha_1+1)}.
\end{equation}
The direct-solve proxy used in Proposition~\ref{prop:compute-optimal} corresponds to $r=2$, recovering Eq.~\eqref{eq:compute-opt-ND}. An accelerated first-order method in the same branch would instead have $r=1+\frac{\alpha_1+\alpha_2+2}{2}$,
which changes the numerical compute exponent but leaves the balanced allocation
$D^*\sim (N^*)^{\alpha_1+1}$ unchanged.

This distinction is especially useful in the mixed regime discussed above, where the cheapest solver can switch across the $(N,D)$ plane. In that regime, acceleration is cheaper than the direct solve only in sufficiently model-heavy regions. The crossover in the intermediate regime occurs at $N_c(D)\sim D^{\frac{\alpha_1+\alpha_2+2}{2(\alpha_1+1)}}$.
By contrast, the loss-balanced allocation satisfies $N_{\mathrm{bal}}(D)\sim D^{1/(\alpha_1+1)}$. In the mixed regime with $\alpha_1+\alpha_2>0$, we have $\frac{1}{\alpha_1+1} < \frac{\alpha_1+\alpha_2+2}{2(\alpha_1+1)}$,
so $N_{\mathrm{bal}}(D)\ll N_c(D)$ asymptotically. Therefore the compute-optimal allocation lies on the direct-solve side of the crossover, even though acceleration is cheaper in more heavily overparameterized regions.

For the representative experimental choice $(\alpha_1,\alpha_2)=(1,0.3)$, this comparison gives
$N_{\mathrm{bal}}(D)\sim D^{1/2}$ and $N_c(D)\sim D^{0.825}$. Thus the balanced allocation is well below the accelerated/direct crossover. Consequently, even under a piecewise cheapest-solver compute model, the asymptotic compute-optimal allocation is governed by the same direct-solve branch used in Proposition~\ref{prop:compute-optimal}. More generally, alternative solver models can modify the compute exponent $\alpha_C$, but they do not alter the qualitative conclusion that sparsity shifts the compute-optimal allocation toward data.

\subsection{Compute-Optimal Allocation}
Proposition~\ref{prop:compute-optimal} adopts the proxy compute budget in Eq.~\eqref{eq:proxy},
$C = ND \cdot \min\{N,D\}$, and shows that the unique compute-optimal allocation lies in the underparameterized regime, $N<D$. The optimal allocation and resulting loss scale as
\begin{equation}
\begin{split}
\ell^*(C) &\sim C^{-\alpha_C}, \\
N^*(C) &\sim C^{\alpha_D/(\alpha_N+2\alpha_D)}
= C^{1/(\alpha_1+3)}, \qquad
D^*(C) \sim C^{\alpha_N/(\alpha_N+2\alpha_D)}
= C^{(\alpha_1+1)/(\alpha_1+3)},\\
\alpha_C &:= \frac{\alpha_N\alpha_D}{\alpha_N+2\alpha_D}
= \frac{\alpha_1+\alpha_2+1}{\alpha_1+3} > 0 .
\end{split}
\label{eq:compute-opt-ND-app}
\raisetag{4ex}
\end{equation}

\subsubsection{Sketch of Proposition~\ref{prop:compute-optimal}}
\label{app:optimal-derivation}

\begin{proof}[Sketch (compute-optimal scaling)]
We optimize the loss under the compute model
\[
C = ND\cdot \min\{N,D\},
\]

\paragraph{Underparameterized Side ($N<D$).}
Here the compute constraint becomes
\[
C = ND\cdot N = N^2D.
\]
Assuming the population loss when $\alpha_1>0$ takes the form
\[
\ell(N,D)\sim N^{-\alpha_N}+D^{-\alpha_D},
\]
we eliminate $D$ using the compute constraint, $D = \frac{C}{N^2}$:
\[
\ell(C,N)\;\sim\; N^{-\alpha_N}+\left(\frac{C}{N^2}\right)^{-\alpha_D}
\;=\; N^{-\alpha_N}+C^{-\alpha_D}N^{2\alpha_D}.
\]
For fixed $C$, the first term decreases with $N$ while the second increases
with $N$ (since larger $N$ forces smaller $D=C/N^2$). Thus the optimum balances
the two contributions. Equating their scalings,
\[
N^{-\alpha_N}\;\asymp\; C^{-\alpha_D}N^{2\alpha_D}
\quad\Longleftrightarrow\quad
N^{\alpha_N+2\alpha_D}\;\asymp\; C^{\alpha_D},
\]
yields
\[
N^*(C)\;\asymp\; C^{\frac{\alpha_D}{\alpha_N+2\alpha_D}},
\qquad
D^*(C)=\frac{C}{(N^*(C))^2}\;\asymp\; C^{1-\frac{2\alpha_D}{\alpha_N+2\alpha_D}}.
\]
At this optimum, both loss terms scale equally, so
\[
\ell^*(C)\;\sim\; (N^*(C))^{-\alpha_N}
\;\asymp\;
C^{-\alpha_C},
\qquad
\alpha_C=\frac{\alpha_N\alpha_D}{\alpha_N+2\alpha_D},
\]
which gives the exponents reported in Eq.~\eqref{eq:compute-opt-ND-app}.

\paragraph{Overparameterized Side ($D<N$).}
In this case $\min\{N,D\}=D$ and the compute constraint becomes $C=ND^2$.
Optimizing under this branch produces a candidate scaling that is not
self-consistent with $D<N$, and is therefore ruled out by Remark~\ref{Remark:ComputeOverParam}, so the unique compute-optimal allocation lies in the
underparameterized regime.

\end{proof}

\section{Gradient Descent}\label{appendix:GD}

Gradient descent (GD) provides a natural procedure for minimizing the training loss. In favorable settings, it converges to the trained-to-completion min-norm solution, though rare failures can occur due to sparsity as shown in the main text. Below, we provide a proof of Proposition~\ref{thm:gd_eta2_hpg} and a proof sketch for Theorem~\ref{thm:gd_failure_scaling}, which together summarize our main findings on GD, as reported in the main text.

\subsection{Proof of Proposition~\ref{thm:gd_eta2_hpg}}\label{subsec:thm:gd_eta2_hpg}
\begin{proof}[Proof (spectral bound for stability)]
Under the near-isometry approximation on $\mathrm{span}(\mathbf{X})$, we may replace
$\mathbf{u}^\top\mathbf{u}$ by $\mathbf{I}_M$ inside $\mathbf{u}^\top\mathbf{u}\,\mathbf{X}\mathbf{X}^\top$,
so it suffices to control $\lambda_{\max}\!\big(\tfrac{1}{D}\mathbf{X}\mathbf{X}^\top\big)$.

If empirical second moments concentrate, then $\tfrac{1}{D}\mathbf{X}\mathbf{X}^\top\approx\boldsymbol{\Sigma}_x$
with $\boldsymbol{\Sigma}_x=\mathbb{E}[\mathbf{x}\mathbf{x}^\top]$ diagonal and
$\mathrm{Var}(x_j)=j^{-(\alpha_1+\alpha_2+2)}$, whose top eigenvalue equals $\mathrm{Var}(x_1)=1$. Although $\boldsymbol{\Sigma}_x$ is diagonal, $\tfrac{1}{D}\mathbf{X}\mathbf{X}^\top$ need not be. The off-diagonal entries have mean zero and concentrate entrywise; we assume that their aggregate contribution is subleading, so they do not change the leading $O(1)$ scale of the top eigenvalue governed by the diagonal part. The only potential obstruction is an anomalously large diagonal entry caused by a single rare activation at a very large index when $\alpha_2<-1$ (so activated amplitudes grow with $j$).

Let $j_{\max}$ be the largest index that activates at least once in the dataset.
For a given threshold $J$, the expected number of activations at indices larger than $J$ is
\[
D\sum_{j>J} \mathbb{P}(x_j\neq 0)
= D\sum_{j>J} j^{-(\alpha_1+1)}
\;\asymp\; \frac{1}{\alpha_1}D\,J^{-\alpha_1}.
\]
Equivalently, writing
\[
N_{>J}:=\sum_{d=1}^D\sum_{j>J}\mathbf{1}\{x^{(d)}_j\neq 0\},
\]
a Poisson approximation yields
\[
\mathbb{P}(N_{>J}=0)\approx\exp\!\left(-\frac{D}{\alpha_1}J^{-\alpha_1}\right),
\qquad
\mathbb{P}(j_{\max}\le J)=\mathbb{P}(N_{>J}=0)\approx\exp\!\left(-\frac{D}{\alpha_1}J^{-\alpha_1}\right),
\]
and hence
\[
\mathbb{P}(j_{\max}=J)\approx D\,J^{-(\alpha_1+1)}\exp\!\left(-\frac{D}{\alpha_1}J^{-\alpha_1}\right).
\]
In particular, $j_{\max}\asymp D^{1/\alpha_1}$ with high probability.

Conditional on an activation at $j_{\max}$, the corresponding diagonal contribution to
$\tfrac{1}{D}\mathbf{X}\mathbf{X}^\top$ is of order
\[
\frac{1}{D}x_{j_{\max}}^2
=\frac{1}{D}j_{\max}^{-(\alpha_2+1)}
\sim  D^{-(\alpha_1+\alpha_2+1)/\alpha_1}.
\]
But, given that $\alpha_2<-1$, and since $\alpha_1+\alpha_2+1>0$, we have $D^{-(\alpha_1+\alpha_2+1)/\alpha_1}<1$, so such rare spikes cannot dominate the
$O(1)$ contribution from the leading coordinates (in particular $j=1$). Therefore
$\lambda_{\max}\!\big(\tfrac{1}{D}\mathbf{X}\mathbf{X}^\top\big)=1+o_{\mathbb{P}}(1)$, and the same holds for
$\lambda_{\max}\!\big(\tfrac{1}{D}\mathbf{u}^\top\mathbf{u}\,\mathbf{X}\mathbf{X}^\top\big)$ under near-isometry.
The conclusion of Proposition~\ref{thm:gd_eta2_hpg} then follows from the usual spectral stability condition on the step size.
\end{proof}

\subsection{Sketch of Theorem~\ref{thm:gd_failure_scaling}}

\begin{proof}[Sketch (rare-activation spike)]
We focus on the event that a single ``late'' coordinate activation produces a spiked top eigenvalue.
Write the data matrix as $\mathbf{X}=[\,\mathbf{X}_{\rm reg}\ \ \mathbf{x}_{\rm spec}\,]$, where
$\mathbf{X}_{\rm reg}\in\mathbb{R}^{M\times(D-1)}$ contains the regular samples and
$\mathbf{x}_{\rm spec}\in\mathbb{R}^{M}$ is a special sample containing a single unusually large activation.
Then
\[
\mathbf{X}\mathbf{X}^\top
=
\mathbf{X}_{\rm reg}\mathbf{X}_{\rm reg}^\top+\mathbf{x}_{\rm spec}\mathbf{x}_{\rm spec}^\top.
\]
In a spiked covariance scenario, the top eigenvector aligns with $\mathbf{x}_{\rm spec}$ and the
corresponding eigenvalue is bounded below by the Rayleigh quotient
\begin{equation}
\lambda_{\max}(\mathbf{X}\mathbf{X}^\top)
\geq
\frac{\mathbf{x}_{\rm spec}^\top \mathbf{X}\mathbf{X}^\top \mathbf{x}_{\rm spec}}{\|\mathbf{x}_{\rm spec}\|_2^2}
=
\frac{\mathbf{x}_{\rm spec}^\top \mathbf{X}_{\rm reg}\mathbf{X}_{\rm reg}^\top \mathbf{x}_{\rm spec}}{\|\mathbf{x}_{\rm spec}\|_2^2}
+\|\mathbf{x}_{\rm spec}\|_2^2.
\label{eq:rayleigh_spike}
\end{equation}

\JSCheck{Under the finite-variance condition $\alpha_1+\alpha_2+1>0$, the first term in \eqref{eq:rayleigh_spike} is  $o(D)$ at the rare-spike scale}, so the condition for GD instability,
$\lambda_{\max}\!\left(\mathbf{X}\mathbf{X}^\top\right)>\frac{2D}{\eta}$, is to leading order in $D$
equivalent to requiring
\begin{equation*}
\|\mathbf{x}_{\rm spec}\|_2^2 \;\gtrsim\; \frac{2D}{\eta}.
\label{eq:norm_condition}
\end{equation*}
In our model, a single-coordinate activation at index $j$ has magnitude
$\|\mathbf{x}_{\rm spec}\|_2^2\approx x_j^2=j^{-(\alpha_2+1)}$. Hence \eqref{eq:norm_condition} requires
\begin{equation}
j \;\gtrsim\; j^\ast
:=
\left(\frac{2D}{\eta}\right)^{\frac{1}{-\alpha_2-1}}
\qquad(\alpha_2<-1).
\label{eq:j_star}
\end{equation}
Finally, let $N_{>j^\ast}$ denote the (random) number of activations with index larger than $j^\ast$ in the
dataset. By definition of $j^\ast$, the rare-spike divergence event is the existence of at least one activation with index $>j^\ast$, so $\mathbb{P}_{\mathrm{rare}}(\mathrm{diverge})\asymp \mathbb{P}(N_{>j^\ast}\geq 1)$.

The expected number of activations at indices larger than $j^\ast$ is
\[
\mathbb{E}[N_{>j^\ast}]
=
D\sum_{j>j^\ast}\mathbb{P}(x_j\neq 0)
=
D\sum_{j>j^\ast} j^{-(\alpha_1+1)}
\;\asymp\;
\frac{1}{\alpha_1}D\,(j^\ast)^{-\alpha_1}.
\]

Equivalently (as in Subsection~\ref{subsec:thm:gd_eta2_hpg}), a Poisson approximation gives
$\mathbb{P}(N_{>j^\ast}\ge 1)=1-\exp(-\mathbb{E}[N_{>j^\ast}])\asymp \mathbb{E}[N_{>j^\ast}]$ when
$\mathbb{E}[N_{>j^\ast}]$ is small, hence
$\mathbb{P}(N_{>j^\ast}\ge 1)\asymp \frac{1}{\alpha_1}D\,(j^\ast)^{-\alpha_1}$.

Substituting~\eqref{eq:j_star} gives
\[
\mathbb{P}_{\mathrm{rare}}(\mathrm{diverge}) \asymp
\frac{1}{\alpha_1}D
\left(\frac{2D}{\eta}\right)^{-\frac{\alpha_1}{-\alpha_2-1}}
\asymp D^{-\frac{\alpha_1+\alpha_2+1}{-\alpha_2-1}},
\]
which is Eq.~\eqref{eq:gd_failure_scaling} up to $\eta$-dependent constants.
\end{proof}

\subsubsection{Empirical Considerations for Theorem~\ref{thm:gd_failure_scaling}}
\label{app:gd_failure_empirical}
We discuss the practical difficulty of empirically validating Theorem~\ref{thm:gd_failure_scaling}'s rare-spike prediction $\mathbb{P}_{\mathrm{rare}}(\mathrm{diverge})\asymp D^{-\nu}$, and report our finite-$D$ measurements, which are consistent with the theorem's qualitative form.

\JSCheck{A clean validation would estimate $\mathbb{P}\!\left(\lambda_{\max}(\mathbf{X}\mathbf{X}^\top/D) > \epsilon\right)$ for fixed threshold $\epsilon = 2/\eta$ across a range of $D$ values spanning the asymptotic regime, then fit the slope on log-log axes and compare to the predicted exponent $\nu = (\alpha_1+\alpha_2+1)/(-\alpha_2-1)$. However, this represents a statistical challenge: the theorem characterizes a tail probability, and at any given $D$, accurate estimation of $\mathbb{P}_{\mathrm{rare}}(\mathrm{diverge})=p$ requires $\Theta(1/p)$ independent samples to observe events at the relevant rate, with sample size scaling further with $\nu$ for accurate exponent estimation. For the theorem's asymptotic regime to dominate observation, $D$ must be large enough that finite-$D$ corrections to the spectral concentration are negligible compared to the leading $D^{-\nu}$ behavior; in our experiments, this asymptotic crossover occurs beyond $D \approx 10^4$. Combining these requirements, a quantitative validation at a predicted exponent $\nu = 9$ (for chosen parameters $\alpha_1 = 2$, $\alpha_2 = -1.2$) would require roughly $10^9$ to $10^{12}$ independent dataset samples per $D$ value---many orders of magnitude beyond computational tractability with eigenvalue computation on $D \times D$ Gram matrices.}

\JSCheck{We perform experiments in which we sweep $D \in \{500, 1000, 2000, 5000, 10000\}$ at fixed threshold $\epsilon = 1.0005$ and feature dimension $M = 5000$, sampling between $500$ and $4000$ independent dataset replicates per $D$. We chose $\alpha_1 = 2,\ \alpha_2 = -1.2$ (giving $\nu = 9$) as a representative point well inside the regime $\alpha_2 < -1$ and $\alpha_1 + \alpha_2 + 1 > 0$ where Theorem~\ref{thm:gd_failure_scaling} applies (smaller values of $\nu$ are more easily detected in finite samples but tend to give less clean separation between tail-event probabilities and bulk concentration). The threshold $\epsilon$ was chosen empirically so that the smallest-$D$ measurement gives a moderate (non-saturated) failure rate. Our results are shown in Table~\ref{tab:gd_failure_empirical}.}

\begin{table}[h]
\vspace{-3mm}
\caption{Empirical divergence probability as a function of dataset size $D$, with $\alpha_1 = 2$, $\alpha_2 = -1.2$, $M=5000$, threshold $\epsilon=1.0005$.}
\label{tab:gd_failure_empirical}
\small
\setlength{\tabcolsep}{4pt}
\renewcommand{\arraystretch}{1.15}
\begin{tabularx}{\columnwidth}{
    >{\centering\arraybackslash}X
    >{\centering\arraybackslash}X
    >{\centering\arraybackslash}X
}
\toprule
$D$ & failures / seeds & $\mathbb{P}(\mathrm{diverge})$ \\
\midrule
500    & $185 / 500$  & $0.370$ \\
1000   & $132 / 1000$ & $0.132$ \\
2000   & $40 / 2000$  & $0.020$ \\
5000   & $1 / 3000$   & $0.00033$ \\
10000  & $0 / 4000$   & $< 0.00025$ \\
\bottomrule
\end{tabularx}
\vspace{-2mm}
\end{table}

\JSCheck{The decay spans more than three orders of magnitude in $\mathbb{P}(\mathrm{diverge})$ over an order of magnitude in $D$. Computing local exponents from successive $D$ pairs gives $\nu_{\mathrm{local}} \approx 1.5,\ 2.7,\ 4.5$, increasing monotonically---consistent with finite-$D$ corrections that vanish as the asymptotic regime is approached.}

\JSCheck{Two observations matter for the comparison to theory. First, the empirical $\mathbb{P}(\mathrm{diverge})$ exhibits a clean, monotone power-law decay with $D$---consistent in functional form and direction with Theorem~\ref{thm:gd_failure_scaling}. Second, the empirical exponent at the largest measurable $D$ ($\nu_{\mathrm{local}} \approx 4.5$) is approaching, but has not reached, the asymptotic prediction $\nu = 9$. Both observations are consistent with the theorem: the asymptotic exponent is a statement about $D \to \infty$, and the local empirical exponent's monotonic increase suggests we are inside the finite-$D$ regime where the asymptotic rate has not yet emerged. Distinguishing the asymptotic exponent $\nu = 9$ from any other large-$\nu$ value (e.g., $\nu = 6$ or $\nu = 12$) requires sample sizes and $D$ ranges substantially beyond what is computationally accessible with this experimental design. We conclude that our measurements are quantitatively non-contradictory with the theorem and qualitatively confirm its prediction of a sharp power-law decay in the divergence probability with dataset size.}

\section{\JSCheck{Experiments}}\label{app:training}


\subsection{Min-norm least-squares solutions via gradient-based optimization}

Our scaling-law theory predicts the test loss of the min-norm random-feature readout as $N$ and $D$ vary. We use two independent methods to recover this solution:
\begin{enumerate}[leftmargin=*,topsep=0pt]
    \item \textbf{Closed-form pseudoinverse.} For the linear regression problem $\min_{\boldsymbol{\theta}} \tfrac{1}{2D}\|\boldsymbol{\Phi}^\top\boldsymbol{\theta}-\mathbf{y}^\top\|_2^2$ on (linear or ReLU) random features, the min-norm least-squares solution is $\boldsymbol{\theta}^\star=(\boldsymbol{\Phi}^\top)^+\mathbf{y}^\top$ (equivalently $\boldsymbol{\theta}^{\star\top}=\mathbf{y}\boldsymbol{\Phi}^+$), computed via SVD-based pseudoinverse with relative tolerance $10^{-10}$.
    
    \item \textbf{Nesterov-accelerated GD with adaptive restart}~\citep{nesterov1983method,odonoghue2015adaptive}. The readout $\theta$ is initialized at zero and trained on the empirical mean-squared error with Nesterov momentum and gradient-based adaptive restart. From zero initialization, the converged solution is provably the same min-norm least-squares solution as the closed-form one.
\end{enumerate}

\begin{wrapfigure}[31]{R}{0.525\textwidth}
\vspace{-0.1cm}
\centering
\includegraphics[width=0.525\textwidth]{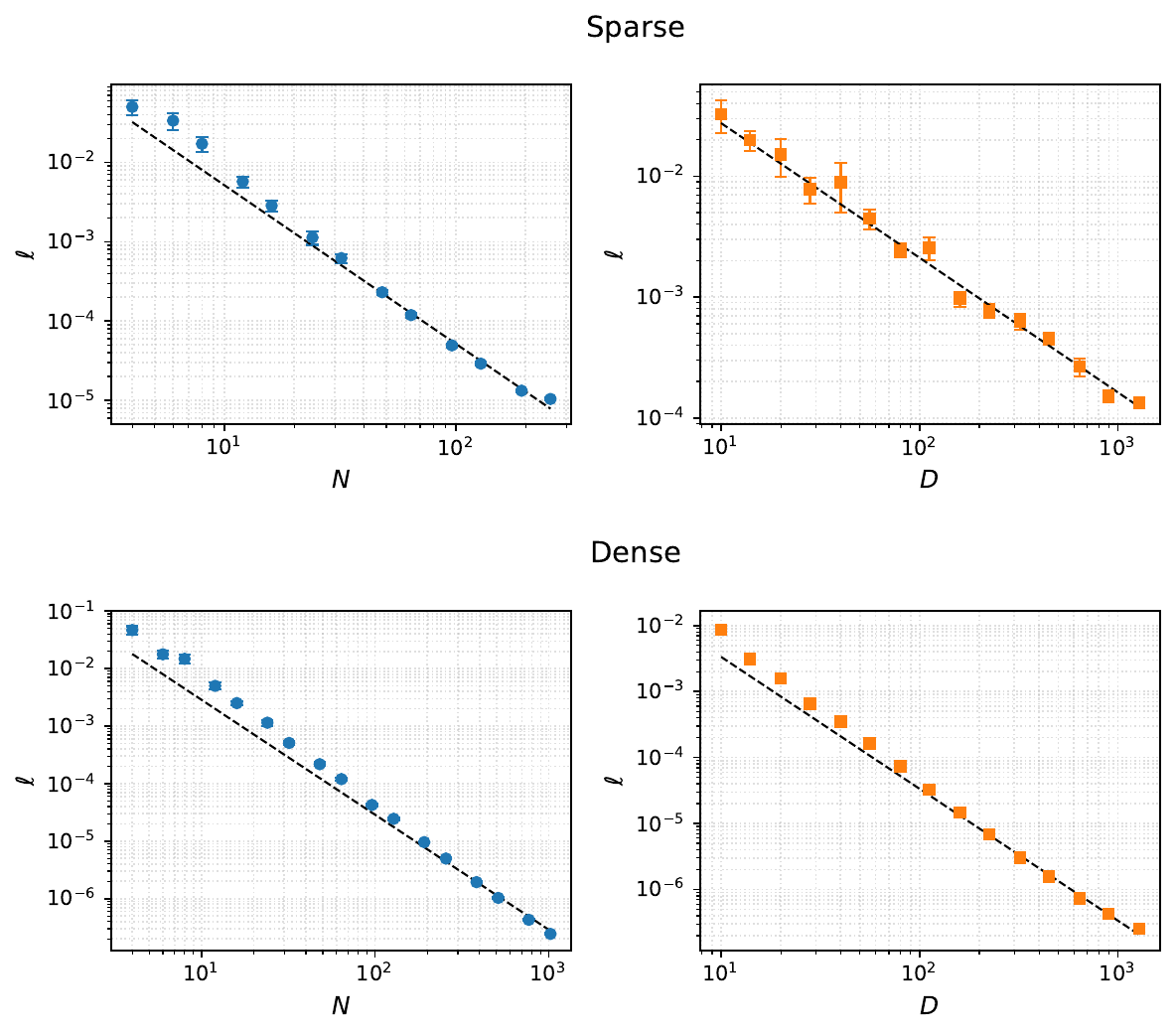}
\vspace{-6mm}
\caption{\textbf{Two-exponent scaling validates linear theory.} Test loss scaling under the linear feature map $\boldsymbol{\phi}(\mathbf{x})=\mathbf{u}\mathbf{x}$, computed via the closed-form min-norm least-squares solution; $(\alpha_1, \alpha_2) = (1.0, 0.3)$, $20$ seeds. Top: sparse; bottom: dense. Left: $N$-sweep at $D = 50{,}000$; right: $D$-sweep at $N = 16{,}000$. A single exponent $\alpha \approx 2.0$ (dashed lines, jointly fitted with separate intercepts) describes the sparse $N$-, dense $N$-, and dense $D$-sweeps, in close agreement with the linear-theory prediction $\alpha_N = \alpha_1 + \alpha_2 + 1 = 2.30$. The sparse $D$-sweep requires its own shallower fit, $\alpha_D \approx 1.11$, matching the predicted sparse exponent $\alpha_D = (\alpha_1+\alpha_2+1)/(\alpha_1+1) = 1.15$ and breaking the dense-baseline symmetry $\alpha_N = \alpha_D$.}
\label{fig:minNormLinear}
\end{wrapfigure}

Figures~\ref{fig:scaling-collapse} and~\ref{fig:compute-scaling} are computed using the closed-form pseudoinverse, while Figure~\ref{fig:MLP} is computed using the
Nesterov-accelerated solver.

The closed-form and Nesterov approaches give the same qualitative two-regime structure, supporting the interpretation that the observed exponent asymmetry is primarily a property of the regression problem rather than of a particular solver.   In the linear case (Figure~\ref{fig:minNormLinear}), the joint exponent across the sparse $N$-sweep, dense $N$-sweep, and dense $D$-sweep recovers $\alpha \approx 2.0$, in close agreement with the linear-theory prediction $\alpha_N = \alpha_1 + \alpha_2 + 1 = 2.30$, while the sparse $D$-sweep gives $\alpha_D \approx 1.11$, matching the theoretical prediction $\alpha_D = (\alpha_1+\alpha_2+1)/(\alpha_1+1) = 1.15$. The asymmetry ratio $\alpha_N/\alpha_D\approx 1.80$ is close to the predicted $2.30/1.15=2.00$, with a residual gap plausibly attributable to finite-size corrections and possible crossover effects near the double-descent peak. Under ReLU (Figures~\ref{fig:minNormReLU} and~\ref{fig:MLP}) both methods give the same qualitative structure: a joint exponent describing three of the four sweeps and a shallower sparse $D$-exponent breaking the dense-baseline symmetry, with the magnitudes $\alpha_N/\alpha_D \approx 1.13$ (closed-form) and $\approx 1.25$ (Nesterov). These finite-size corrections are roughly consistent across the linear and ReLU cases, indicating that nonlinearity reduces the magnitude of the asymmetry but does not alter its qualitative structure.


\begin{wrapfigure}[29]{R}{0.525\textwidth}
\vspace{-0.1cm}
\centering
\includegraphics[width=0.525\textwidth]{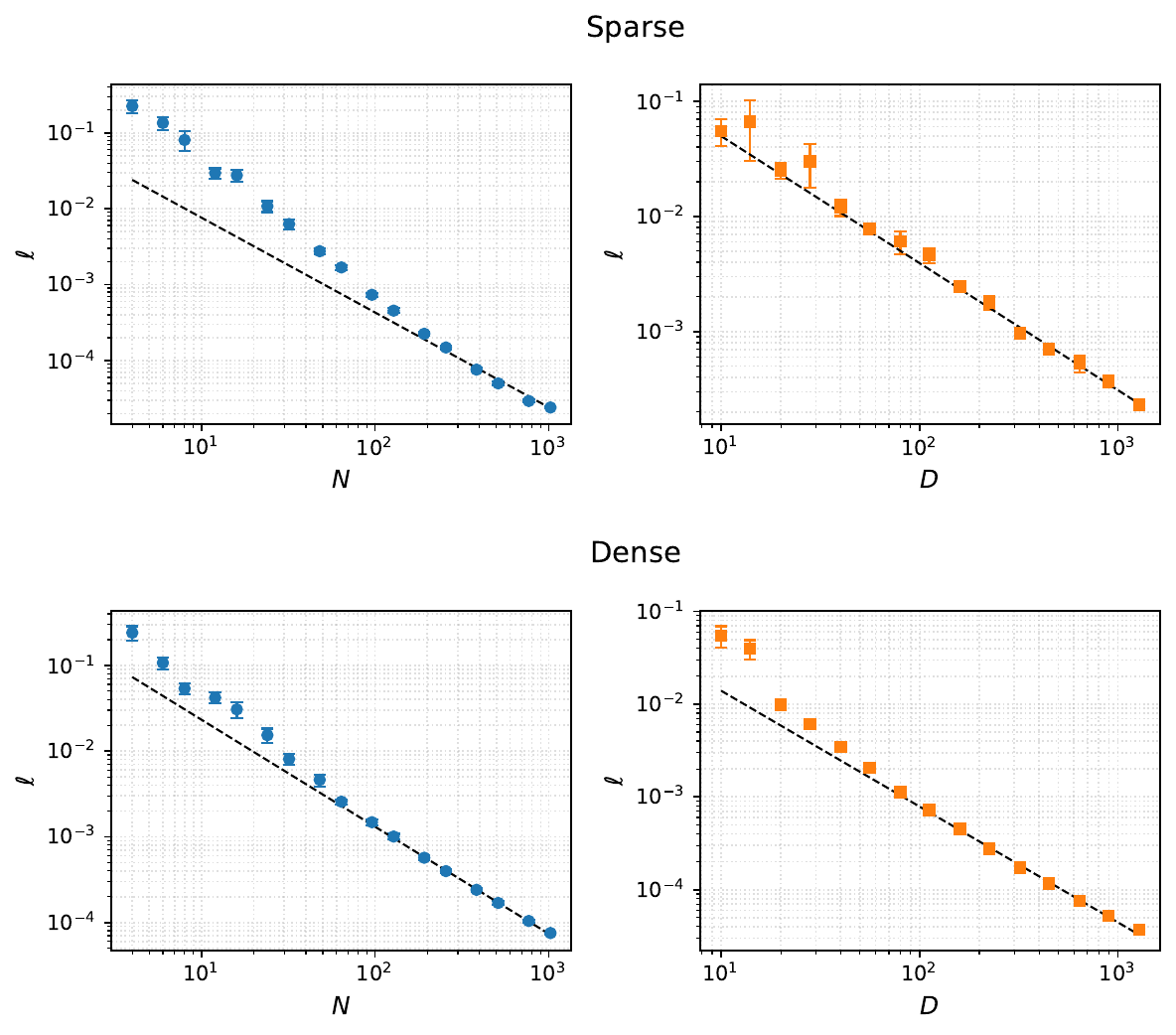}
\vspace{-6mm}
\caption{\textbf{Asymmetry persists under nonlinearity (closed-form).} Test loss scaling under the ReLU feature map $\boldsymbol{\phi}(\mathbf{x})=\sigma(\mathbf{u}\mathbf{x})$, computed via the closed-form min-norm least-squares solution; $(\alpha_1, \alpha_2) = (1.0, 0.3)$, $20$ seeds. Top: sparse; bottom: dense. Left: $N$-sweep at $D = 50{,}000$; right: $D$-sweep at $N = 16{,}000$. A single exponent $\alpha \approx 1.25$ (dashed lines, jointly fitted with separate intercepts) describes the sparse $N$-, dense $N$-, and dense $D$-sweeps. The sparse $D$-sweep requires its own shallower fit, $\alpha_D \approx 1.10$, breaking the dense-baseline symmetry $\alpha_N = \alpha_D$ predicted by linear theory; the smaller magnitude of the asymmetry compared to the linear case (Figure~\ref{fig:minNormLinear}) reflects spectral smoothing under nonlinear feature maps.}
\label{fig:minNormReLU}
\end{wrapfigure}

\paragraph{Setup.} All experiments use the data-generating model defined in Section~\ref{Section:Ourmodel}, with sparse parameters $(\alpha_1, \alpha_2) = (1.0, 0.3)$ and ambient dimension $M = 10{,}000$. The teacher is $y(\mathbf{x}) = \mathbf{w}^\top \mathbf{x}$ with $w_j \sim \mathcal{N}(0, 1)$ i.i.d. The first-layer weights $\mathbf{u} \in \mathbb{R}^{N \times M}$ are drawn i.i.d.\ from $\mathcal{N}(0, 1/N)$ and held frozen throughout training. The dense baseline uses Gaussian inputs with per-coordinate variance $j^{-(\alpha+1)}$, $\alpha = \alpha_1 + \alpha_2 + 1$, matching the underparameterized exponent of the sparse model so that any asymmetry in scaling is attributable purely to sparsity. Test loss is evaluated on $n_{\mathrm{test}}$ fresh samples drawn from the same distribution.

For each $(N, D)$ configuration we run multiple seeds (varying both data sampling and the random first layer) and report mean test loss with $\pm 1$ standard error. Power-law exponents are fitted on log-log axes by least squares, restricted to the asymptotic tail (last $4$--$6$ values of $N$ or $D$) to mitigate finite-size deviations. As a check, we also compare full-batch GD with minibatch SGD at small learning rate; the resulting trajectories and scaling exponents agree closely, suggesting that the sparse scaling laws reported here are not driven primarily by minibatch gradient noise, but by the underlying optimization/statistical structure.

\paragraph{Experiment 1: Linear random features.} The forward map is $\hat y = \theta^\top (\mathbf{u}\mathbf{x})$. The min-norm readout is solved in float64 via \texttt{torch.linalg.pinv} with $\texttt{rtol} = 10^{-10}$ to handle rank-deficient feature matrices robustly. We sweep $N \in \{4,\allowbreak 6,\allowbreak 8,\allowbreak 12,\allowbreak 16,\allowbreak 24,\allowbreak 32,\allowbreak 48,\allowbreak 64,\allowbreak 96,\allowbreak 128,\allowbreak 192,\allowbreak 256\}$ at fixed $D = 50{,}000$, and $D \in \{10,\allowbreak 14,\allowbreak 20,\allowbreak 28,\allowbreak 40,\allowbreak 56,\allowbreak 80,\allowbreak 112,\allowbreak 160,\allowbreak 224,\allowbreak 320,\allowbreak 448,\allowbreak 640,\allowbreak 896,\allowbreak 1280\}$ at fixed $N = 16{,}000$, with $20$ seeds per configuration and $n_{\mathrm{test}} = 50{,}000$. Fits use the last $6$ points of each sweep.

\paragraph{Experiment 2: ReLU random features.} The forward map is $\hat y = \theta^\top \sigma(\mathbf{u}\mathbf{x})$ with $\sigma$ the ReLU activation. Two solver variants are used:

\subparagraph{\emph{Closed-form (pinv).}} Identical to Experiment 1, with ReLU applied to the features before the pseudoinverse. Same sweep grids, $20$ seeds, and $n_{\mathrm{test}} = 50{,}000$.

\subparagraph{\emph{Iterative (Nesterov).}} The ReLU features $\boldsymbol{\Phi} = \sigma(\mathbf{u}\mathbf{X}) \in \mathbb{R}^{N \times D}$ are precomputed once in float64 on GPU. The step size is set to $\eta = 1/\lambda_{\max}(\boldsymbol{\Phi}\boldsymbol{\Phi}^\top/D)$, where $\lambda_{\max}$ is estimated by $50$ steps of power iteration. Nesterov updates with O'Donoghue--Cand{\`e}s adaptive restart~\citep{odonoghue2015adaptive} are applied until the relative gradient norm satisfies $\|\nabla\| < 10^{-9}\|\nabla_0\|$, with a maximum of $5 \times 10^5$ iterations. We sweep $N \in \{4, 8, 16, 32, 64, 128, 256\}$ at fixed $D = 50{,}000$, and $D \in \{10, 20, 40, 80, 160, 320, 640, 1280\}$ at fixed $N = 8000$, with $5$ seeds per configuration and $n_{\mathrm{test}} = 20{,}000$. Fits use the last $4$ points of each sweep.

\paragraph{Hardware and reproducibility.} All experiments were run on a single NVIDIA T4 or A100 GPU (Google Colab). Total wall-clock time was approximately $4$ hours for the closed-form sweeps (Experiments 1 and 2 closed-form) and $30$ minutes for the Nesterov sweeps. Float64 precision is used throughout the optimization and pseudoinverse computations; data sampling and feature evaluation use float32 with TF32 matrix multiplications enabled. The first layer $\mathbf{u}$ is regenerated per seed.

\end{document}